\documentclass[11pt]{article}

\usepackage[preprint]{acl}

\usepackage{times}
\usepackage{latexsym}

\usepackage[T1]{fontenc}

\usepackage[utf8]{inputenc}

\usepackage{microtype}

\usepackage{inconsolata}

\usepackage{graphicx}
\usepackage{amsmath}
\usepackage{amssymb}
\usepackage{mathtools}
\usepackage{amsthm}

\usepackage{booktabs}
\usepackage{tcolorbox}
\usepackage{multirow}
\usepackage{standalone}
\usepackage{makecell}
\usepackage[table]{xcolor}
\usepackage{algorithm}
\usepackage{algorithmic}
\usepackage{graphicx}
\usepackage{listings}
\usepackage{multicol}
\usepackage{graphicx}
\usepackage{subcaption}
\usepackage{siunitx}
\usepackage{microtype}
\usepackage{graphicx}
\usepackage{subcaption}
\usepackage{booktabs} 
\theoremstyle{plain}
\newtheorem{theorem}{Theorem}[section]

\theoremstyle{definition}
\newtheorem{definition}[theorem]{Definition}

\theoremstyle{remark}
\usepackage{amsmath}

\usepackage{hyperref}
\usepackage{booktabs}
\usepackage{tabularx}
\usepackage{enumitem}
\usepackage[textsize=tiny]{todonotes}

\usepackage[capitalize]{cleveref}

\title{Causal Probing for Internal Visual Representations in \\Multimodal Large Language Models}

\author{
    \mdseries
    Zehao Deng\textsuperscript{1,2$\dagger$}\thanks{Work done during Zehao Deng's visit at Shanghai Jiao Tong University and internship at Ant Group. $\dagger$ Equal Contribution. $\ddagger$ Corresponding authors. This work was supported by the Natural Science Foundation of Shanghai (24ZR1440300) and Ant Group Research Fund.}, 
    Tianjie Ju\textsuperscript{1$\dagger$}, 
    Zheng Wu\textsuperscript{1}, 
    Liangbo He\textsuperscript{2},\\
    Jun Lan\textsuperscript{2$\ddagger$},
    Huijia Zhu\textsuperscript{2},
    Weiqiang Wang\textsuperscript{2},
    Zhuosheng Zhang\textsuperscript{1$\ddagger$}
    \vspace{2mm} \\ 
    \textsuperscript{1}School of Computer Science, Shanghai Jiao Tong University \hspace{0.2em}
    \textsuperscript{2}Ant Group \\
    {\ttfamily\small fqyb1122@gmail.com, \{jometeorie,wzh815918208,zhangzs\}@sjtu.edu.cn} \\
    {\ttfamily\small \{liangbo.hlb, yelan.lj, huijia.zhj, weiqiang.wwq\}@antgroup.com}
}

\begin{document}
\maketitle

\begin{abstract}
Despite the remarkable success of Multimodal Large Language Models (MLLMs) across diverse tasks, the internal mechanisms governing how they encode and ground distinct visual concepts remain poorly understood. To unravel these mechanisms, we propose a causal framework based on activation steering to actively probe and manipulate internal visual representations. Through systematic intervention across four visual concept categories, our results reveal a divergence in concept encoding: entity knowledge is distinctively localized, whereas abstract concepts are globally distributed across the network. Critically, this divergence uncovers a mechanistic driver of scaling laws: increasing model depth is indispensable for encoding distributed and complex abstract concepts, whereas entities maintain a consistently high degree of localization. Furthermore, reverse steering uncovers that blocking explicit output triggers a surge in latent activations, exposing a compensatory mechanism between perception and generation. Finally, by extending our analysis to visual reasoning, we expose a disconnect between perception and reasoning: although MLLMs successfully recognize geometric relations, they treat them merely as static visual features, failing to trigger the procedural execution necessary for solving problems.
\end{abstract}

\section{Introduction}
Multimodal Large Language Models (MLLMs) have demonstrated remarkable performance across various tasks such as object recognition, image understanding, and multimodal reasoning~\cite{wu2023multimodal,yin2024survey,zhang2025igniting}. These successes raise a fundamental question: \textbf{How do MLLMs encode different visual concepts?} It is difficult to ascertain the underlying mechanism through common behavioral evaluation or correlation analysis, as correct outputs often hide a reliance on shallow heuristics rather than internal representations \cite{agrawal2018don,mccoy2019right,zhang2022paradox}. 

\begin{figure}[t]
    \centering
    \includegraphics[width=\linewidth,page=2]{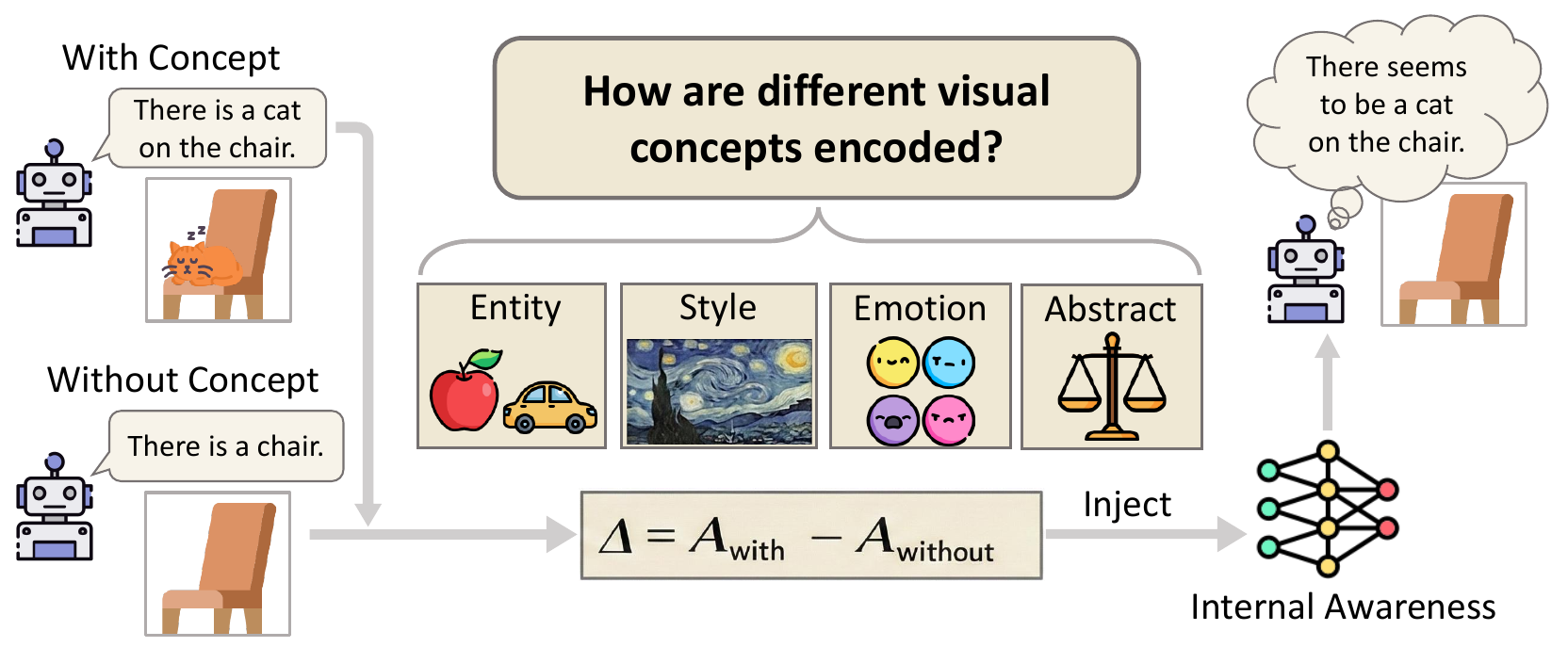}
    \caption{Illustration of our causal framework. We investigate the encoding of visual information by extracting concept vectors $\Delta$ through contrastive activation analysis between paired images (with concept and without concept). By injecting the vector $\Delta$ into the model during inference, we verify whether this intervention causally induces the corresponding internal awareness, thereby revealing the mechanism of visual representations of different concepts in MLLMs.}
    \label{fig:intro}
    \vspace{-0.5em}
\end{figure}

To rigorously understand the model's internal encoding mechanism, we address the challenge above by shifting from passive observation to causal intervention. We employ activation steering \cite{turner2023steering,zou2023representation} to directly manipulate the model’s internal states during inference. By extracting representations of concepts and injecting them into model's residual streams, we can measure the causal influence of these manipulations on the model's behavior.

\begin{table*}[h!]
    \centering
    \footnotesize
    \begin{tabularx}{\textwidth}{lXl}
        \toprule
        \textbf{Category} & \textbf{Explanation} & \textbf{Example} \\
        \midrule
        Entity & 
        Specific and discrete physical objects exist in the image. This is the most basic unit of visual perception, which usually has a clear spatial boundary. & 
        cat, apple \\
        \addlinespace 
        
        Visual Style & 
        Holistic and distributed visual attributes encompassing texture, tone, or artistic genre. These features pervade the entire image representation. & 
        cartoon, Picasso \\
        \addlinespace
        
        Emotion & 
        Affective semantics derived from visual cues. This category requires the model to bridge low-level perception with high-level sentiment analysis. & 
        happiness, disgust \\
        \addlinespace
        
        Abstract Concept & 
        High-level meanings lack direct counterparts. Recognizing these requires model to perform complex inference, mapping visual signals to broader semantic spaces. & 
        justice, danger \\
        \bottomrule
    \end{tabularx}
    \caption{Visual concept classification. We categorize visual concepts into four distinct types: Entity, Visual Style, Emotion, and Abstract Concept. The table defines each category along with examples used in our dataset.}
    \label{tab:classification}
\end{table*}

We first analyze the representational extent, sparsity, and distribution of different visual concepts in different MLLMs. Then we use concept erasure to test whether these vectors are causally necessary for generation. Finally, we investigate the abstract concepts underlying visual logical reasoning using geometric auxiliary lines explore the representation boundary of current MLLMs.

\paragraph{Key Insights} Through systematic intervention and evaluation, we reveal some key findings, summarized as follows:
\begin{itemize}[noitemsep, topsep=1pt,leftmargin=*]
    \item Entities exhibit distinct localization within specific layers, supporting the hypothesis that factual knowledge is encapsulated as explicit key-value pairs. In contrast, abstract concepts manifest as globally distributed representations, implying that their manipulation requires a holistic intervention strategy rather than local editing (\S \ref{sec:3.1}).
    \item Larger models expand the distribution of abstract concepts' representation, utilizing their increased depth to progressively encode complex semantics, which are deficient in small models. It can be an explanation for the emergence of advanced capabilities when models are scaled up (\S \ref{sec:3.1}).
    \item Different concepts exhibit distinct layer-wise distributions, for example, emotions are highly concentrated in the intermediate layers while abstract concepts skew towards terminal layers. Furthermore, these distributions exhibit intra-model-family similarity alongside inter-model-family divergence, indicating that structure and training strategy dictate the encoding patterns. (\S \ref{sec:3.2}).
    \item Reverse steering by subtracting concept can suppress explicit text generation, but it would trigger a compensation in latent activations to reconcile with visual evidence (\S \ref{sec:4}).
    \item While current MLLMs can recognize the geometric relations of auxiliary lines, they lack the intrinsic mechanism to recognize these visual features as signals for problem-solving. It highlights a disconnection between visual perception and logical reasoning (\S \ref{sec:5}).
\end{itemize}

\paragraph{}
In summary, our contributions are as follows:

\begin{itemize}[noitemsep, topsep=1pt,leftmargin=*]
    \item We curate a dedicated dataset that isolates specific visual conceptual differences, which provides a reliable foundation for causal analysis.
    \item We conduct a systematic analysis of MLLMs' internal representations by establishing a comprehensive evaluation framework. This framework evaluates causal effects to explore the encoding mechanism of different visual concepts.
    \item We identify distinct encoding mechanisms from representational extent, sparsity, and distribution for different visual concepts in MLLMs, and uncover the role of scaling laws in representing and understanding complex concepts.  
    \item We reveal current models' limitations in processing abstract concepts, especially visual logical reasoning, which may provide a direction for improving MLLMs' capabilities.
\end{itemize}

\begin{figure*}[t]
    \centering
    \includegraphics[width=0.98\linewidth,page=1]{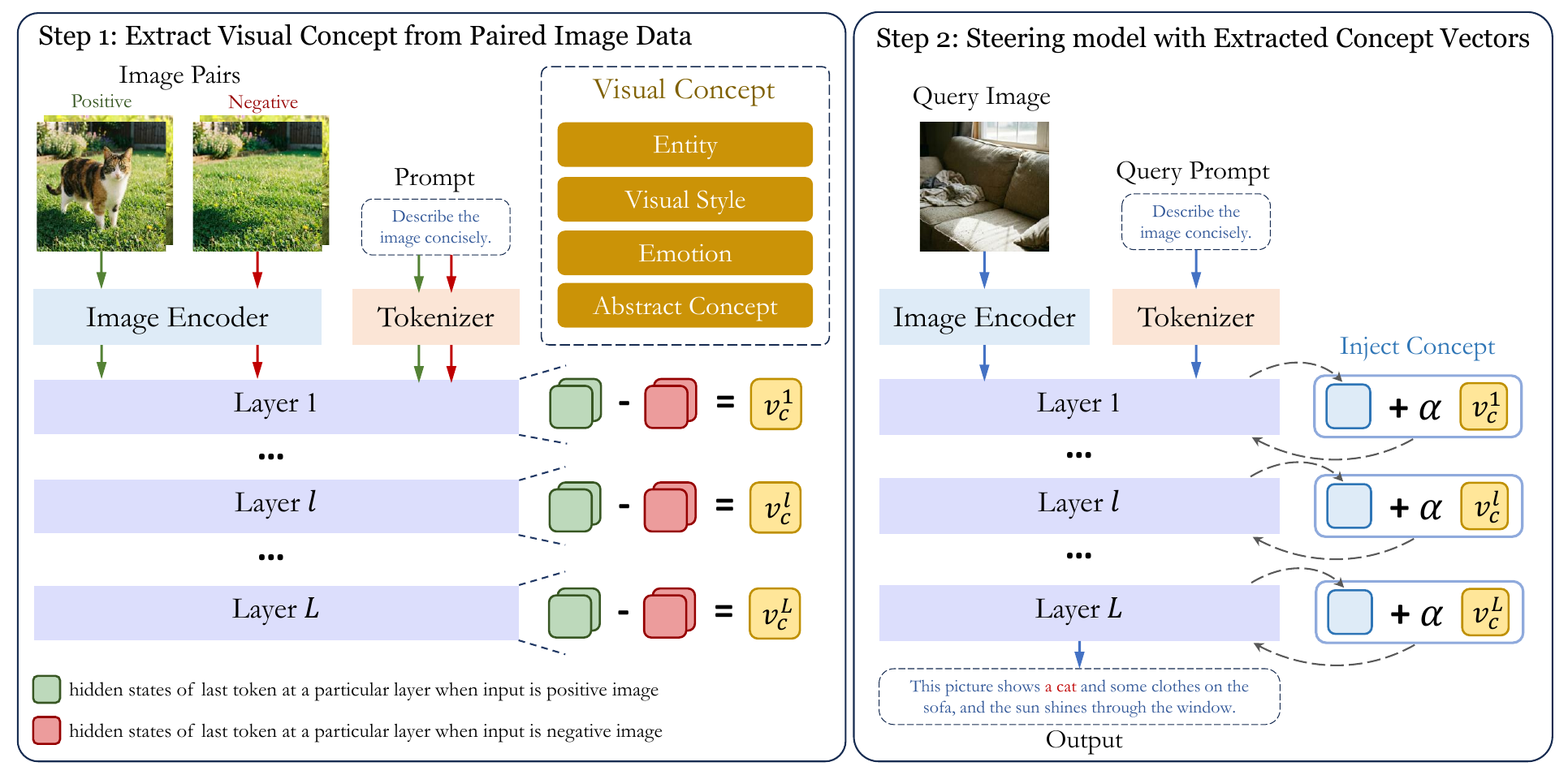}
    \caption{Overview of our activation steering methodology. In step 1, we extract the concept vector $v_c^{l}$ by inputting positive and negative image respectively, and computing mean shifts between hidden states. In step 2, this vector is injected independently into the model's residual stream to modulate the generation process.}
    \label{fig:model}
\end{figure*}

\section{Methodology}
\label{sec:2}
\subsection{Concept Classification}
We want to explore how MLLMs internally encode visual concepts. We define the visual concept as a semantic factor, which can be separated from the visual input of the representation layer, and intervene to induce consistent changes in the model output. Specifically, we classify visual concepts into the following four categories: Entity, Visual Style, Emotion, and Abstract Concept. Detailed explanation and examples are shown in \Cref{tab:classification}.

\subsection{Activation Steering}
Activation steering is based on the linear representation hypothesis \cite{park2023linear,nanda2023emergent,zou2023representation}. As shown in \Cref{fig:model}, activation steering mainly involves two stages: (i) extracting vectors of specific concepts from activation values, (ii) then steering the model's residual stream with concept vectors during inference, to guide the model to output the expected content.

\paragraph{Extract concept.} For a specific concept $c$, such as cat, we construct $N$ image pairs ($x_{c,i}^+$, $x_{c,i}^-$), $i \in \{1, ...,N\}$, where $x_{c,i}^+$ denotes positive image and $x_{c,i}^-$ denotes negative image. We compute the difference-in-means of the corresponding activations to get the target concept. Formally, we extract concept vector $v_c^{l} \in \mathbb{R}^{d_{model}}$ at the layer $l \in \{1,...,L\}$ as:
\begin{equation}
    v_c^{l} = \frac{1}{N} \sum_{i=1}^N{ \left( h^{l}(x_{c,i}^+)-h^{l}(x_{c,i}^-) \right) }, 
\end{equation}
where $h^{l}(x_{c,i}^+)$ represents the model's hidden states or residual stream activation value of the last token at layer $l$, with the image input $x_{c,i}^+$. Appendix \ref{app:proof}
provides a mathematical proof of why we use the simple arithmetic operation (differences-in-means) to extract optimal concepts.

\paragraph{Steering.} After extracting the target concept vector, we intervene in the model during forward propagation during inference to verify the causal efficacy of $v_c^{l}$ and probe the model's cognitive architecture. For a given query input $x_q$, we modify the internal activation $h^{l}(x_q)$ by injecting the steering vector scaled by a coefficient $\alpha$:
\begin{equation}
    h^{l}(x_q) = h^{l}(x_q) + \alpha \cdot v_c^{l}.
\end{equation}
By varying $\alpha$ and the injection layer $l$, and observing the shifts in the model's generation, we can quantify the functional role of different visual concepts within MLLM.

\begin{table*}[t]
\centering\small
\setlength{\tabcolsep}{7.2pt}
\resizebox{\textwidth}{!}{%
\begin{tabular}{ll w{c}{1.1cm} w{c}{1.1cm} w{c}{1.1cm} w{c}{1.1cm} S[table-format=2.1e2,retain-zero-exponent=true] w{c}{1.1cm}}
\toprule
\multirow{2}{*}{\textbf{Category}} & \multirow{2}{*}{\textbf{Model}} & \multicolumn{2}{c}{\textbf{Success Rate} $\uparrow$} & \multicolumn{2}{c}{\textbf{Semantic Similarity} $\uparrow$} & \multicolumn{2}{c}{\textbf{Logit Boost} $\uparrow$} \\
\cmidrule(lr){3-4} \cmidrule(lr){5-6} \cmidrule(lr){7-8}
 & & \textbf{Peak} & \textbf{DUI} & \textbf{Peak} & \textbf{DUI} & {\textbf{Peak}} & \textbf{DUI} \\
\midrule
\multirow{7}{*}{Entity} 
 & Qwen2.5-VL-7B & 0.637 & 0.051 & 0.461 & 0.022 & 5.0e10 & 0.946 \\
 & Qwen3-VL-8B & 0.610 & 0.089 & 0.371 & 0.020 & 5.7e7 & 0.910 \\
 & Qwen3-VL-32B & 0.647 & 0.083 & 0.412 & 0.017 & 2.3e11 & 0.979 \\
 & LLaVA-OneVision-1.5-8B & 0.395 & 0.092 & 0.373 & 0.041 & 2.5e7 & 0.845 \\
 & Gemma3-4B & 0.596 & 0.012 & 0.414 & 0.013 & 5.2e21 & 0.971 \\
 & Gemma3-27B & 0.639 & 0.099 & 0.418 & 0.031 & 8.1e17 & 0.943 \\
 \rowcolor{gray!10} \cellcolor{white}& \textbf{Average} & \textbf{0.587} & \textbf{0.071} & \textbf{0.408} & \textbf{0.024} &{-} & \textbf{0.932} \\
\midrule
\multirow{7}{*}{Visual Style} 
 & Qwen2.5-VL-7B & 0.215 & 0.253 & 0.443 & 0.059 & 3.9e5 & 0.964 \\
 & Qwen3-VL-8B & 0.291 & 0.164 & 0.444 & 0.048 & 8.5e10 & 0.970 \\
 & Qwen3-VL-32B & 0.562 & 0.157 & 0.582 & 0.040 & 2.5e10 & 0.967 \\
 & LLaVA-OneVision-1.5-8B & 0.118 & 0.488 & 0.449 & 0.053 & 2.1e9 & 0.970 \\
 & Gemma3-4B & 0.154 & 0.523 & 0.431 & 0.064 & 5.5e0 & 0.940 \\
 & Gemma3-27B & 0.200 & 0.443 & 0.560 & 0.053 & 2.5e10 & 0.981 \\
 \rowcolor{gray!10} \cellcolor{white} & \textbf{Average} & \textbf{0.257} & \textbf{0.338} & \textbf{0.485} & \textbf{0.053} &{-} & \textbf{0.965} \\
\midrule
\multirow{7}{*}{Emotion} 
 & Qwen2.5-VL-7B & 0.790 & 0.232 & 0.514 & 0.022 & 5.6e2 & 0.819 \\
 & Qwen3-VL-8B & 0.556 & 0.525 & 0.527 & 0.099 & 1.4e6 & 0.861 \\
 & Qwen3-VL-32B & 0.673 & 0.272 & 0.542 & 0.019 & 8.6e2 & 0.857 \\
 & LLaVA-OneVision-1.5-8B & 0.730 & 0.394 & 0.634 & 0.096 & 6.2e8 & 0.924 \\
 & Gemma3-4B & 0.900 & 0.306 & 0.637 & 0.020 & 9.6e5 & 0.971 \\
 & Gemma3-27B & 0.873 & 0.314 & 0.654 & 0.058 & 5.1e8 & 0.962 \\
 \rowcolor{gray!10} \cellcolor{white} & \textbf{Average} & \textbf{0.754} & \textbf{0.341} & \textbf{0.585} & \textbf{0.052} &{-}& \textbf{0.899} \\
\midrule
\multirow{7}{*}{\shortstack[l]{Abstract Concept}} 
 & Qwen2.5-VL-7B & 0.244 & 0.415 & 0.306 & 0.009 & 2.4e1 & 0.839 \\
 & Qwen3-VL-8B & 0.174 & 0.299 & 0.203 & 0.014 & 2.6e1 & 0.933 \\
 & Qwen3-VL-32B & 0.160 & 0.411 & 0.214 & 0.018 & 1.3e2 & 0.983 \\
 & LLaVA-OneVision-1.5-8B & 0.276 & 0.732 & 0.313 & 0.144 & 3.6e7 & 0.957 \\
 & Gemma3-4B & 0.304 & 0.307 & 0.463 & 0.020 & 9.2e-5 & 0.970 \\
 & Gemma3-27B & 0.360 & 0.410 & 0.477 & 0.024 & 2.2e6 & 0.972 \\
 \rowcolor{gray!10} \cellcolor{white} & \textbf{Average} & \textbf{0.253} & \textbf{0.429} & \textbf{0.329} & \textbf{0.038} & {-} & \textbf{0.942} \\
\bottomrule
\end{tabular}%
}%
\caption{Comparison of steering effects across varying categories and models. For each metric, \textbf{Peak} denotes the maximum value observed across all layers, and \textbf{DUI} denotes the Distribution Uniformity Index, which quantifies the representational span of the effect.}
\label{tab:main_results}
\end{table*}

\subsection{Experimental Setup}
\paragraph{Data.} The data needed in our experiment are mainly image pairs that can reflect the difference of target concepts. For entities, we use the instruction-guided image editing dataset Imgedit \cite{ye2025imgedit} and MagicBrush \cite{Zhang2023MagicBrush}, and build 10,000 image pairs through automated scripts. For style, we use the paired data Omniconsistency \cite{song2025omniconsistency}, which contains 22 visual styles. For emotion, we use Emoset \cite{yang2023emoset}, which not only covers the emotion of facial expression, but also involves various images such as objects, landscapes, scenes, etc. For the abstract concept, we select pictures from Pexels and construct a dataset containing 20 abstract concepts. More details can be found in Appendix \ref{app:data}.

\paragraph{Metrics.} We employ three complementary metrics to assess the intervention: \textbf{(i) Success Rate}, which measures the frequency with which the steering vector successfully elicits the explicit generation of the target concept. \textbf{(ii) Semantic Similarity}, which evaluates the semantic alignment between the generated text and the concept of injection by calculating sentence-BERT
scores. \textbf{(iii) Logit Boost}, which quantifies the magnitude of the causal intervention at the internal representation level. For each metric, we report the \textbf{Peak} value to capture the maximum effect achieved at a certain layer, and the \textbf{Distribution Uniformity Index (DUI)}, calculated as $\text{DUI}=1-\text{Gini}$, to quantify the spatial span of the representation. More details are provided in Appendix \ref{app:metric}. 

\paragraph{Evaluation.} To ensure robust generalization across different architectures and scales, we experiment with the following models: Qwen2.5-VL 7B \cite{bai2025qwen2}, Qwen3-VL 8B, 32B \cite{bai2025qwen3}, Gemma3 4B, 27B \cite{team2025gemma}, and LLaVA-OneVision-1.5 8B \cite{an2025llava}. We use $N=100$ image pairs to extract concepts by difference-in-means. We set the steering coefficient $\alpha=1$, which naturally applies the intrinsic magnitude of the semantic shift observed in the activation space. We test 5 thousand samples when steering and report their average results. The detailed list of prompts is provided in Appendix \ref{app:prompt}.

\section{Probing Internal Representations}
\label{sec:3}

\subsection{Analysis of Steering Effects}
\label{sec:3.1}
\Cref{tab:main_results} presents a systematic comparison of activation steering effects across four categories. We evaluate three dimensions: causality (Logit Boost), semantic (Semantic Similarity), and behavior (Success Rate). For each dimension, we measure the Peak value to evaluate the model's representational capacity and the DUI to determine whether the representation signal is distributed widely or concentrated in specific layers.

\paragraph{Peak Value.} 
We first examine the Peak values in Success Rate and Semantic Similarity to evaluate the representational strength of different concepts. As shown in \Cref{tab:main_results}, the data reveals a distinctive hierarchy where emotion achieves the highest representational upper bound. Specifically, emotion concepts consistently yield the strongest intervention effects, peaking at a Success Rate of 0.900 and a Similarity of 0.637 in Gemma3-4B. This high steerability suggests that \textbf{emotion features are encoded in highly separable latent subspaces}, allowing them to map directly to specific sentiment tokens. Appendix \ref{app:emo} further validates this fine-grained encoding of different emotional categories.

In stark contrast, \textbf{abstract concepts remain the representation bottleneck.} They consistently exhibit the weakest response, with Success Rates dropping as low as 0.160 in Qwen3-VL-32B and Similarity scores hovering around 0.214. This minimal steerability indicates that abstract semantics lack dedicated, highly-activated subspaces. Instead of being explicitly grounded in the perceptual stream like emotions, abstract concepts likely emerge from complex, late-stage feature aggregation, making their representations inherently resistant to interventions.

\paragraph{DUI.}
Analyzing the DUI reveals divergences in how MLLMs store different types of knowledge. \textbf{Entity knowledge is highly localized.} As shown in \Cref{tab:main_results}, entity concepts exhibit consistently low DUI averaging 0.071, with the Gemma3-4B model reaching an extreme concentration of 0.012. This high localization supports the hypothesis that entity knowledge, due to its concrete and discrete nature, is readily encapsulated as explicit key-value pairs within the feed-forward networks of specific layers \cite{meng2022locating}. Consequently, a steering vector functions as a precise key retrieval, which explains why methods like knowledge editing can successfully pinpoint specific knowledge neurons \cite{dai2022knowledge} for entities.

Conversely, \textbf{abstract concepts exhibit much broader, distributed representations across the network.} The DUI for abstract concepts reaches up to 0.429. This suggests that complex abstract semantics cannot be isolated into sparse key-value pairs, but are instead represented across the broader residual stream. This finding also offers a representational perspective on model depth: deeper networks do not necessarily localize entity knowledge better, but they provide more layers required to represent distributed abstract concepts.

\begin{figure*}[t]
    \centering
    \begin{subfigure}[b]{0.24\textwidth} 
        \centering
        \includegraphics[width=\linewidth]{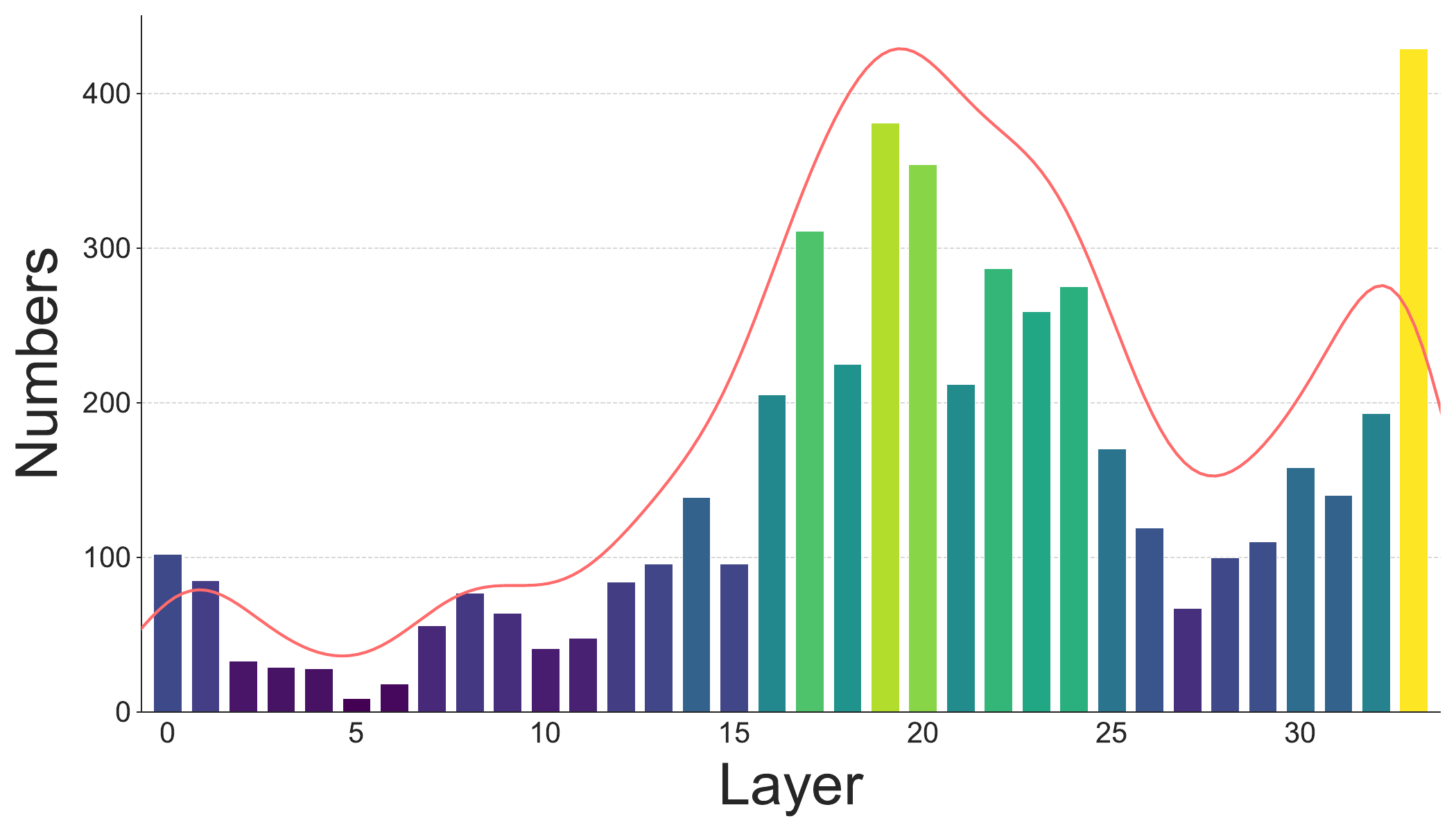} 
        \label{fig:row1-col1}
    \end{subfigure}
    \begin{subfigure}[b]{0.24\textwidth}
        \centering
        \includegraphics[width=\linewidth]{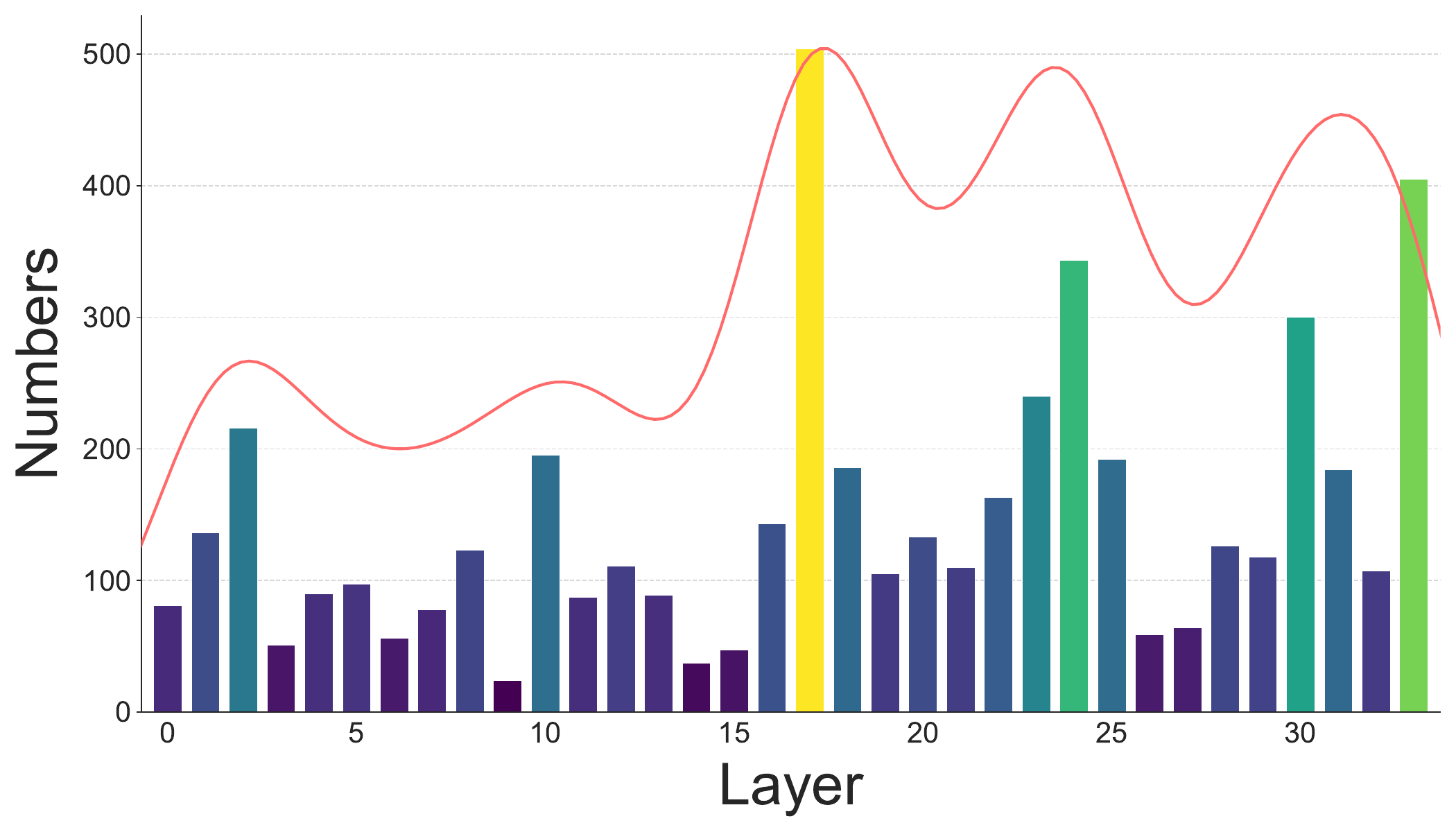}
        \label{fig:row1-col2}
    \end{subfigure}
    \begin{subfigure}[b]{0.24\textwidth}
        \centering
        \includegraphics[width=\linewidth]{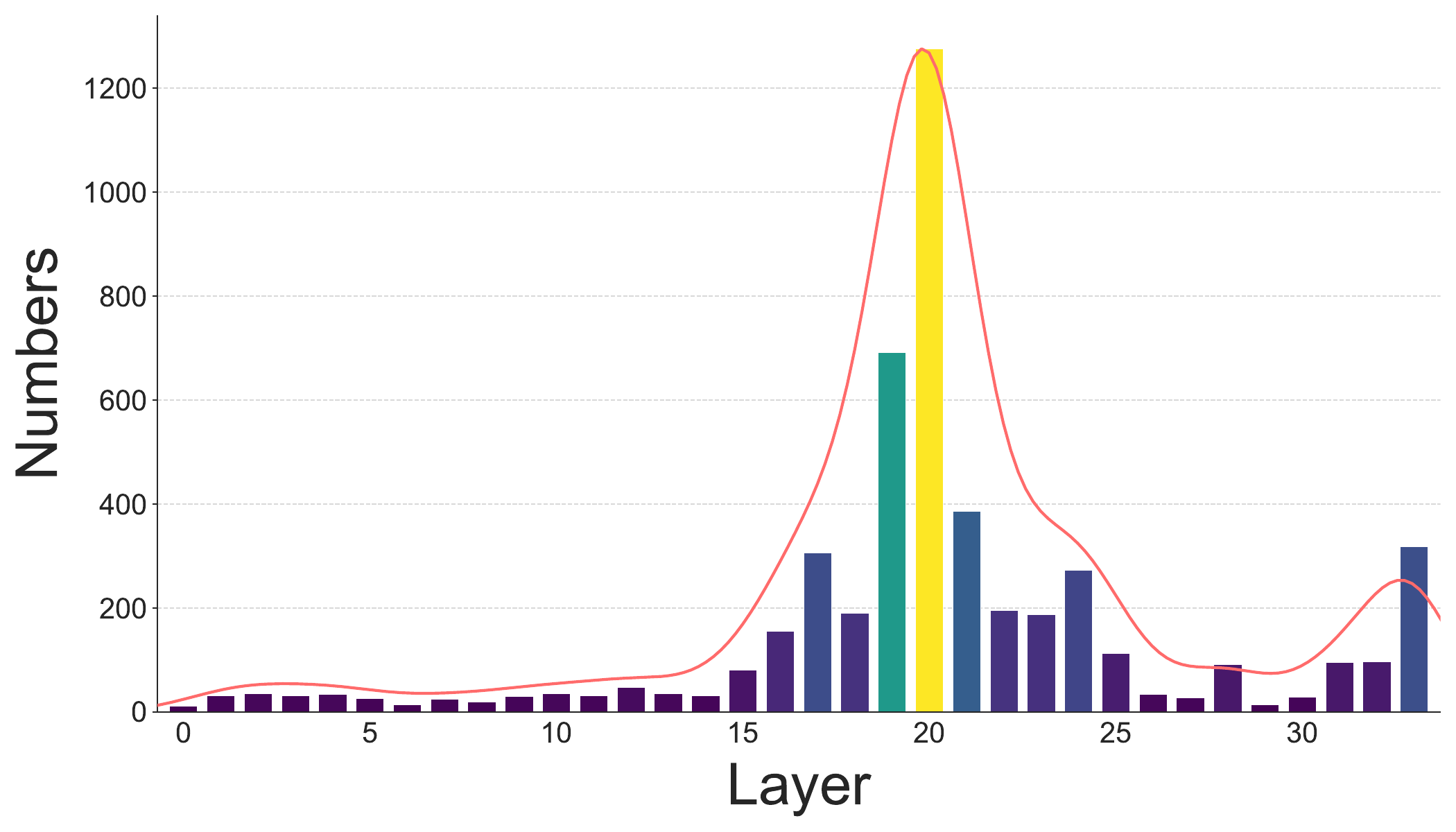}
        \label{fig:row1-col3}
    \end{subfigure}
    \begin{subfigure}[b]{0.24\textwidth}
        \centering
        \includegraphics[width=\linewidth]{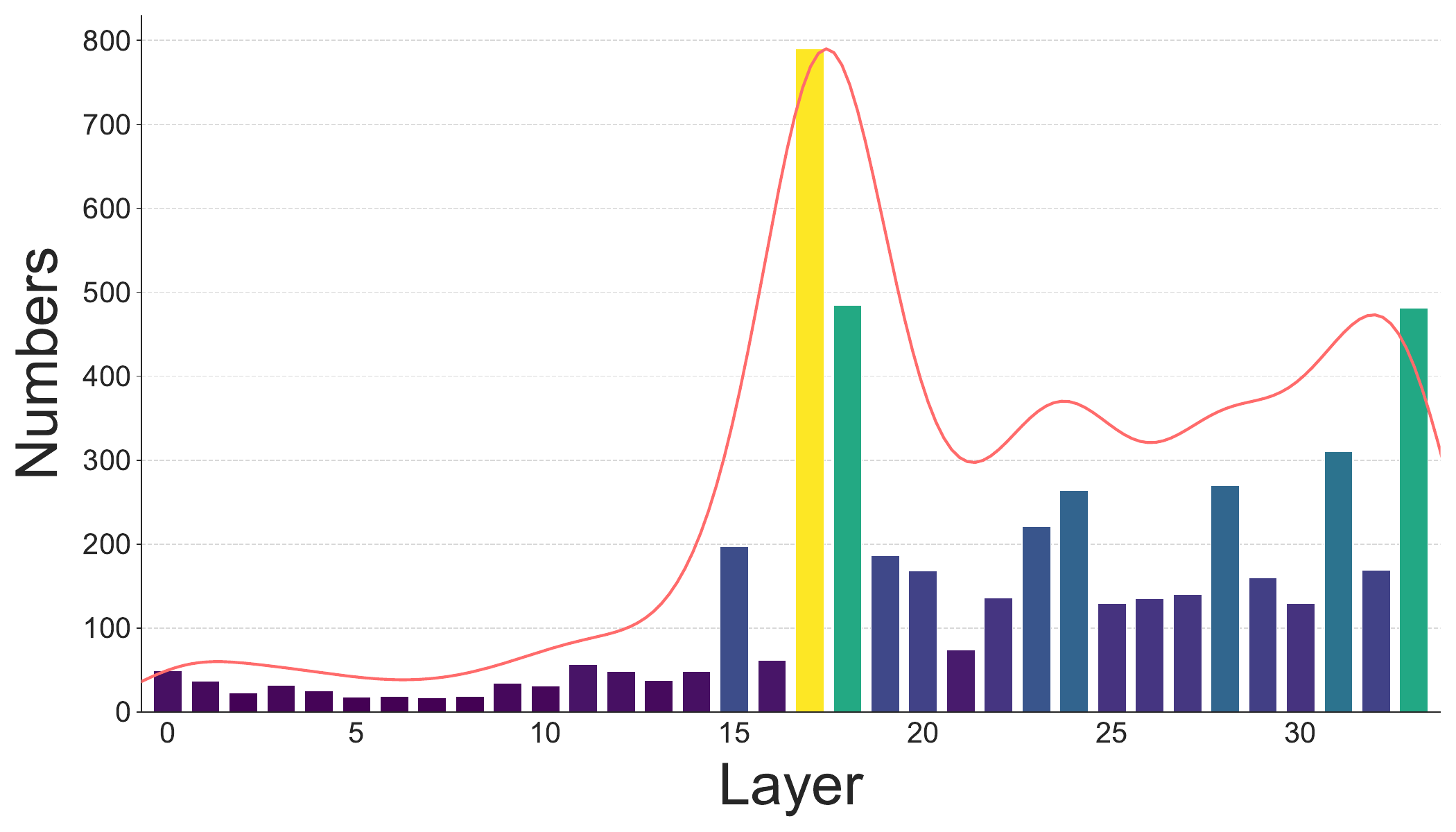}
        \label{fig:row1-col4}
    \end{subfigure}
    
    \vspace{0.5em} 
    
    \begin{subfigure}[b]{0.24\textwidth}
        \centering
        \includegraphics[width=\linewidth]{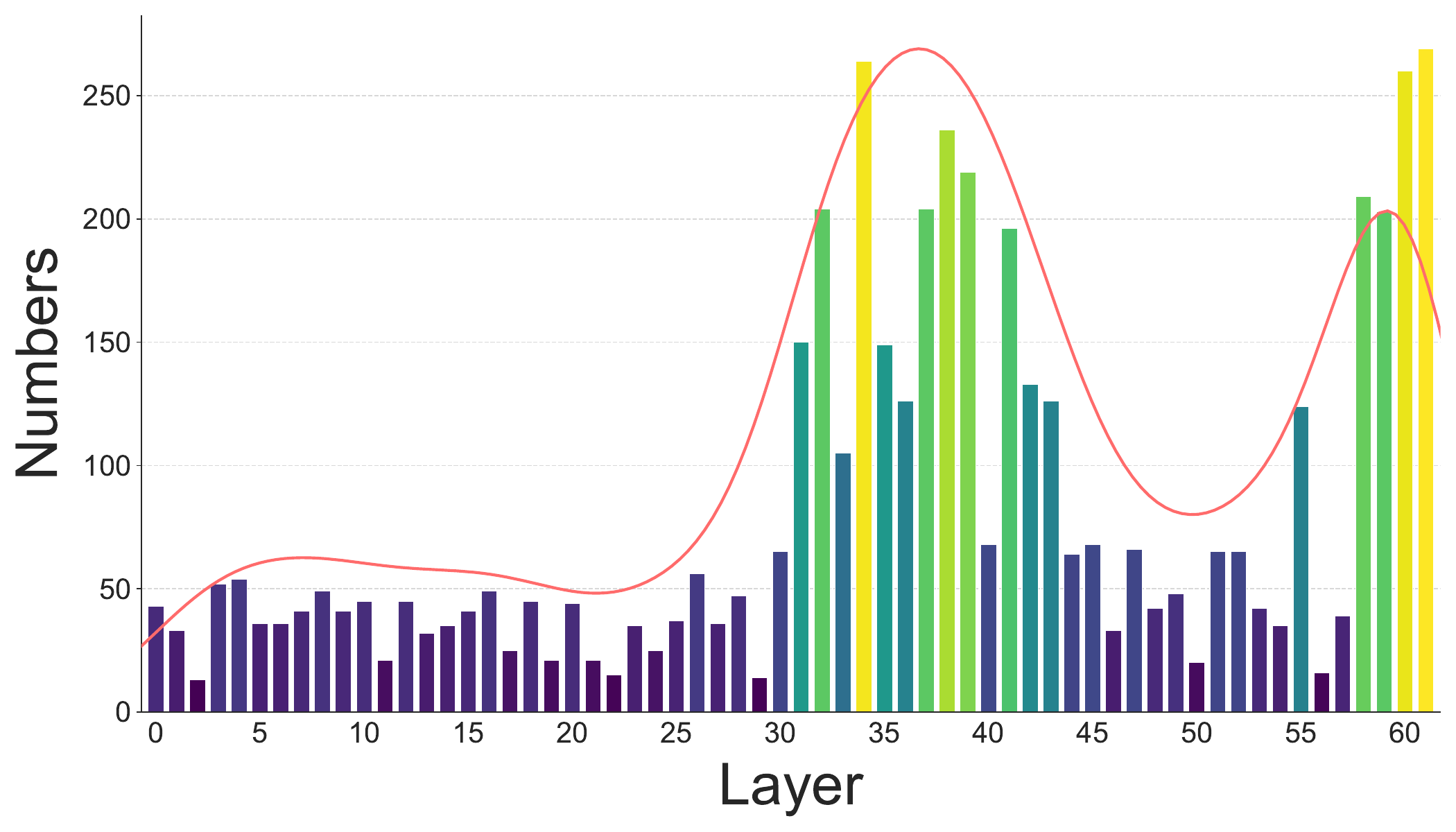}
    \end{subfigure}
    \begin{subfigure}[b]{0.24\textwidth}
        \centering
        \includegraphics[width=\linewidth]{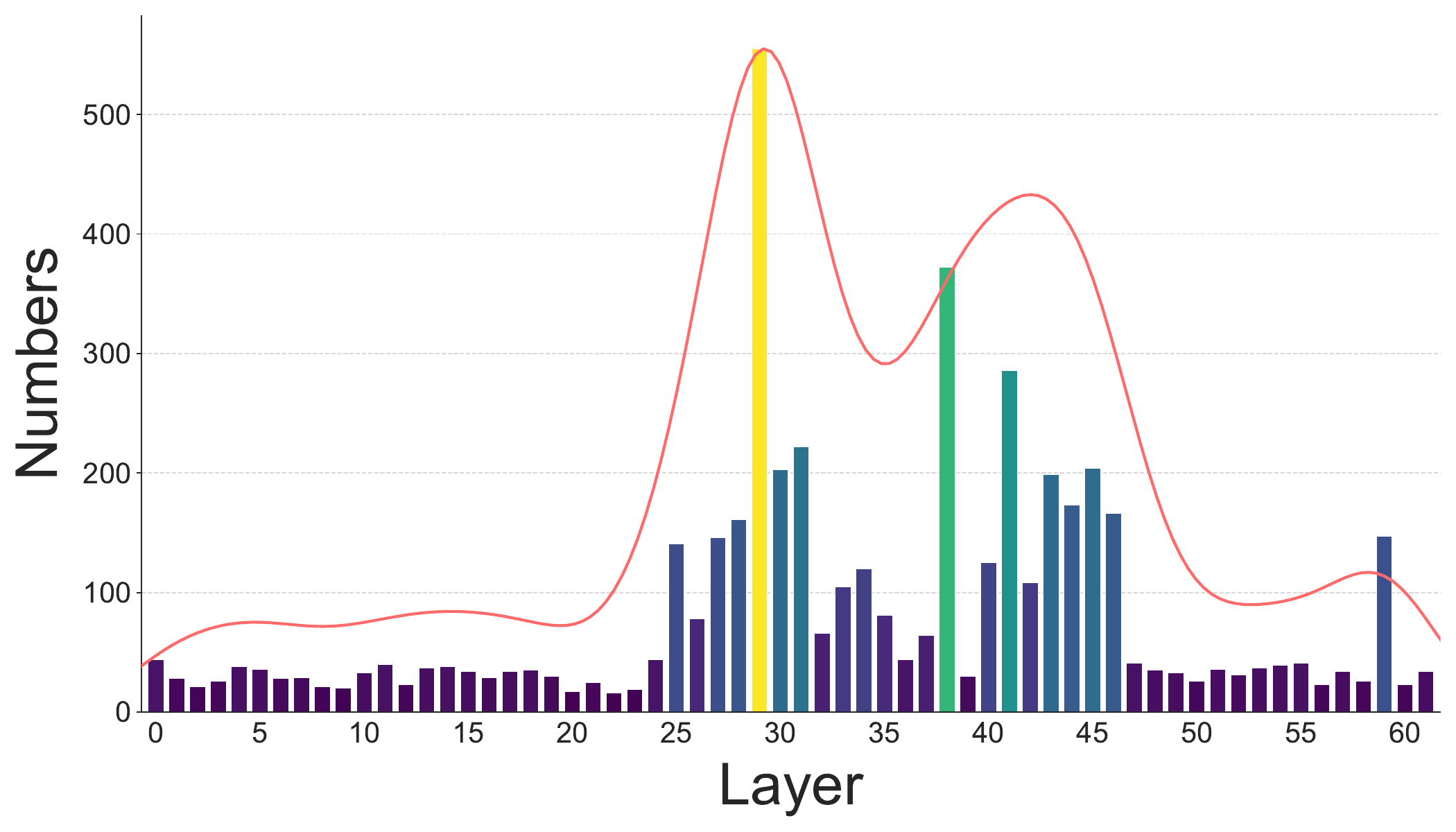}
    \end{subfigure}
    \begin{subfigure}[b]{0.24\textwidth}
        \centering
        \includegraphics[width=\linewidth]{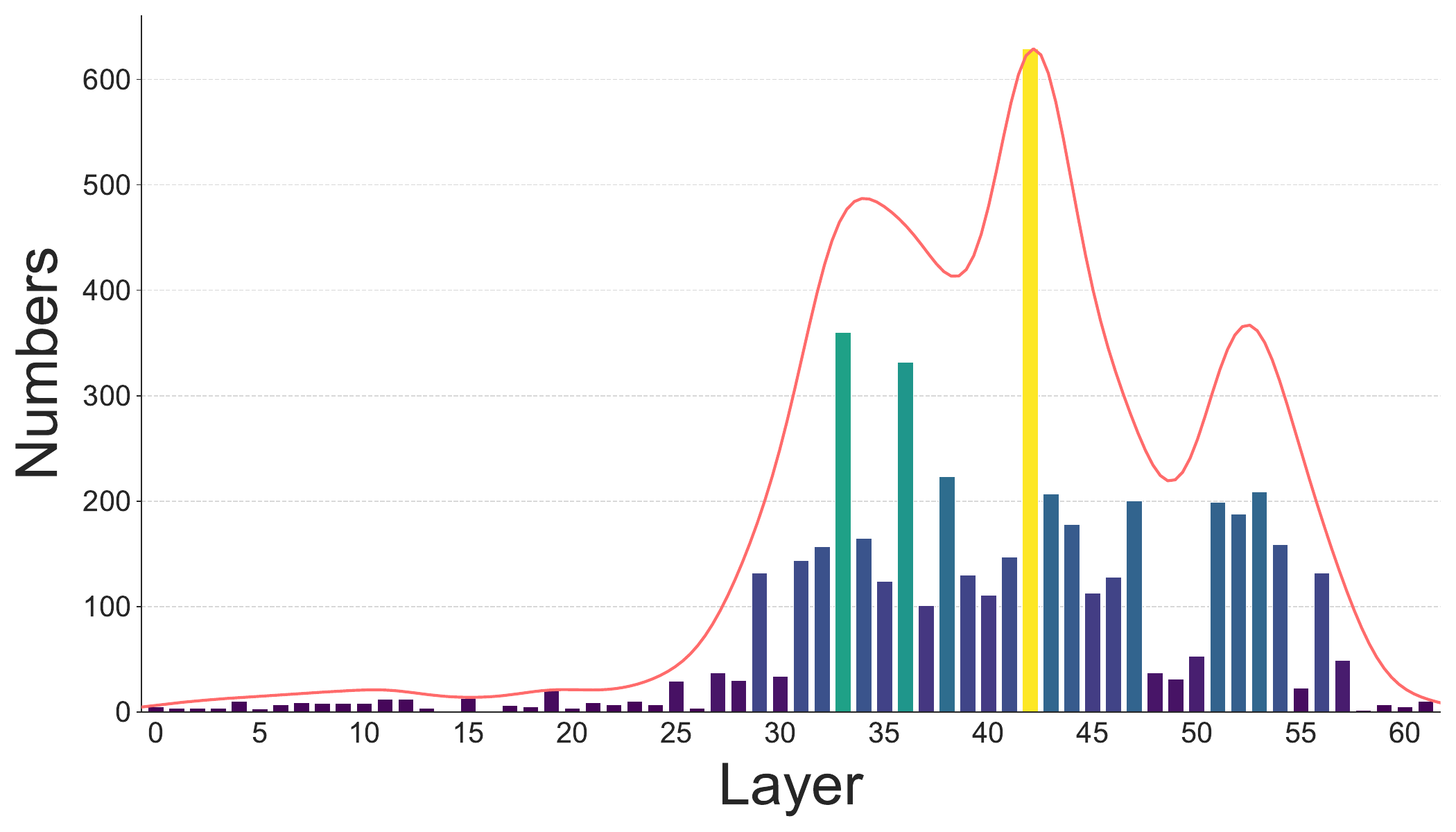}
    \end{subfigure}
    \begin{subfigure}[b]{0.24\textwidth}
        \centering
        \includegraphics[width=\linewidth]{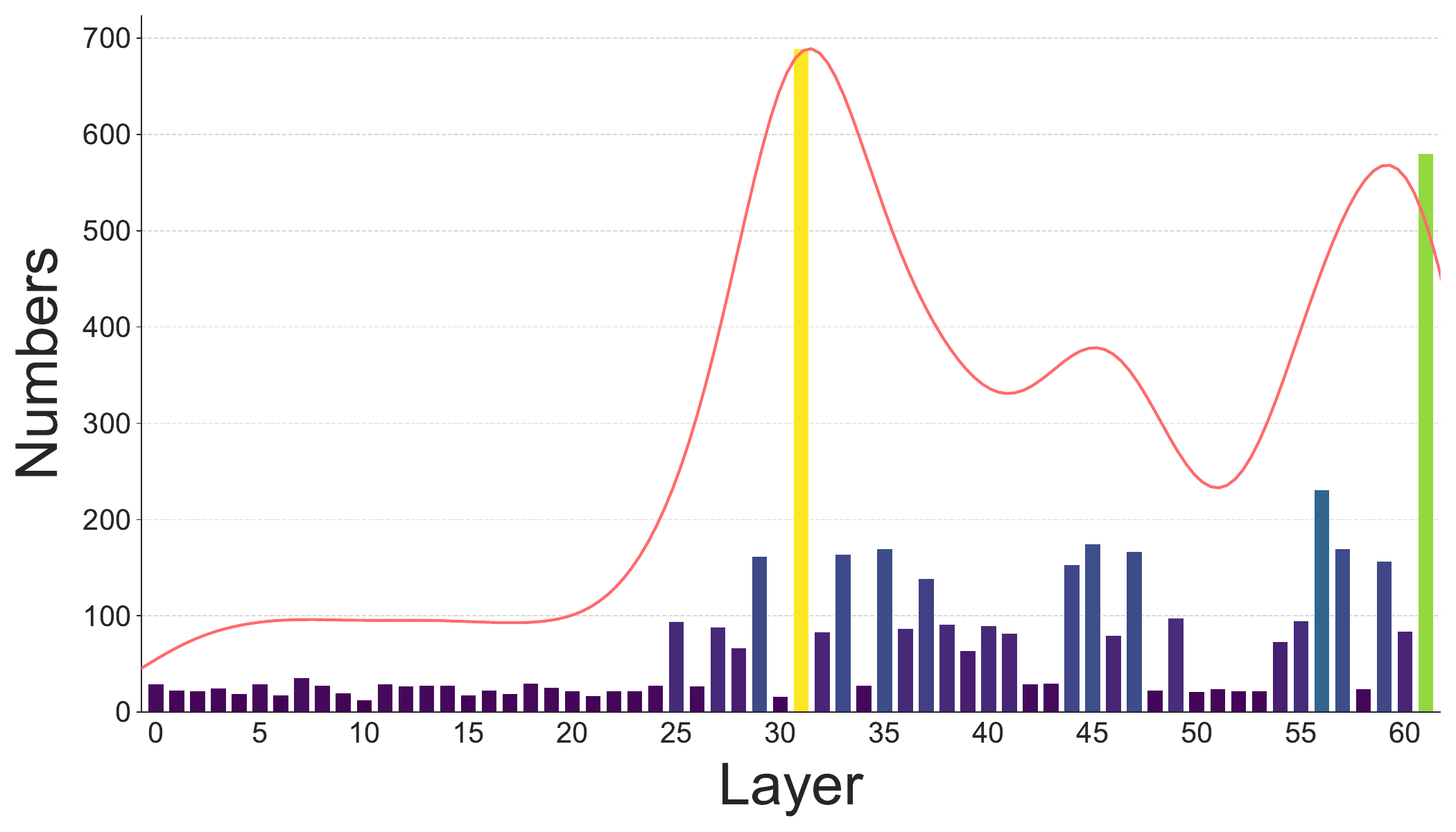}
    \end{subfigure}
    
    \caption{Layer-wise distribution of steering layers. The plots are organized by concept category (columns, left to right: entity, visual style, emotion, and abstract concept) and model scale (rows: top for Gemma3-4B, bottom for Gemma3-27B). Each histogram displays the number of samples that achieved their maximum semantic similarity score at a specific layer. The distributions reveal distinct cognitive signatures: entity shows extensive spatial span and a bimodal structure; visual style exhibits broad dispersion; emotion peaks in intermediate layers; and abstract concept is heavily skewed towards the final reasoning layers.}
    \label{fig:distribution}
\end{figure*}

\paragraph{Logit Boost.}
The Logit Boost metric provides further evidence for these divergent encoding strategies. We observe a massive disparity in intervention magnitudes across categories. Steering with entity concepts induces extreme logit shifts, ranging from $10^{10}$ to $5.2 \times 10^{21}$ in Gemma3-4B. This extreme scale difference suggests that entities are encoded in highly disentangled and localized subspaces within the model's residual stream, where a single directional intervention acts as a forceful control variable independent of other features. However, as analyzed in Appendix \ref{app:context}, extreme logit values do not guarantee explicit generation: entities often encounter decoding suppression due to contextual conflicts, whereas emotion concepts achieve higher Success Rates by smoothly integrating into the context.

In comparison, abstract concepts exhibit negligible boosts, representing a difference of over 10 orders of magnitude. This indicates that abstract concepts lack the dedicated, high-magnitude subspaces characteristic of concrete entities. Their generation relies not on forceful logical switches, but on subtle, accumulated shifts in the probability distribution across the vocabulary space.

\paragraph{Models Scaling.}
Finally, we investigate the impact of model scaling across varying architectures (e.g., the Qwen3-VL and Gemma3 families). We observe that increased model capacity generally yields higher peak steerability. This suggests that larger models develop more robust and well-isolated representations in their latent space, allowing them to handle visual features with reduced semantic interference. More importantly, scaling reveals a divergence in spatial encoding: while larger models tend to exhibit lower DUI for the first three categories (indicating tighter localization), they show significantly higher DUI for abstract concepts.

This finding indicates that larger models provide deeper layers to encode abstract semantics in a distributed manner. This phenomenon aligns with our earlier analysis. Smaller and shallower models likely lack sufficient layers to support the sequential encoding of these complex abstract concepts. In contrast, deeper networks naturally distribute abstract concepts across their extended layers. This offers an explanation for the emergence of advanced capabilities in scaled models: their advantage over smaller models lies not in the stronger memorization of simple knowledge, but in \textbf{utilizing deeper layers to process complex concepts}.

\subsection{Layer-wise Mechanism Analysis}
\label{sec:3.2}

To explore the layer-wise processing mechanism, we calculate the number of samples peaking at each layer and visualize the distribution of steering layers for each category across the Gemma3 4B and 27B models, as shown in \Cref{fig:distribution}. These distributions reveal distinct structural signatures that characterize the layer-wise encoding dynamics of different visual concepts.

\paragraph{Distributions of categories.}
Entity representations exhibit a distinctive bimodal localization. Unlike other categories that are confined to specific regions, entity concepts are confined to specific functional regions. As for bimodality, we attribute the middle peak to visual perception, where the model directly recognizes concrete object features. In contrast, the final peak reflects that this entity concept undergoes continuous accumulation and refinement throughout the network depth. We provide a discussion in Appendix \ref{app:2peak}, demonstrating that \textbf{this bimodal divergence stems from a distinct information stream}.

Visual styles display the broadest and most uniform dispersion across the network. This suggests that stylistic information is not restricted to specific layers or stages; rather, \textbf{it functions as a continuous bias vector} in the residual stream, modulating the generation trajectory across all layers. 

Emotions are highly concentrated in the intermediate layers. This indicates that affective features are primarily mapped to the textual latent space during the intermediate computational phase.

Abstract Concepts exhibit a heavily right-skewed distribution, with optimal intervention points clustered predominantly in the late-stage layers. This distribution confirms that abstract semantics are not directly grounded in the early layers' perceptual stream. Instead, they are synthesized \textbf{as high-level semantic representations in the late layers}, immediately prior to generation.

\paragraph{Architectural Determinants.}
The layer-wise distribution exhibits a critical pattern: \textbf{intra-family similarity alongside inter-family divergence}. Comparing the 4B and 27B parameter scales within the Gemma family reveals that these spatial distribution patterns remain consistent despite the significant difference in capacity. However, these distributions vary considerably across different model families (shown in Appendix \ref{app:arc}). 

This phenomenon indicates that the layer-wise distribution of visual concepts is \textbf{determined by the specific training strategy and structural configuration}. Models within the same family share identical vision encoders, projection modules, and pre-training objectives. Consequently, scaling up parameters expands the network's capacity but preserves the intrinsic computational sequence of feature processing. In contrast, the inter-family divergence demonstrates that different architectural designs alter how and when visual features are integrated into the textual space.

\begin{figure*}[t]
    \centering
    \begin{subfigure}[b]{0.32\textwidth}
        \centering
        \includegraphics[width=\linewidth]{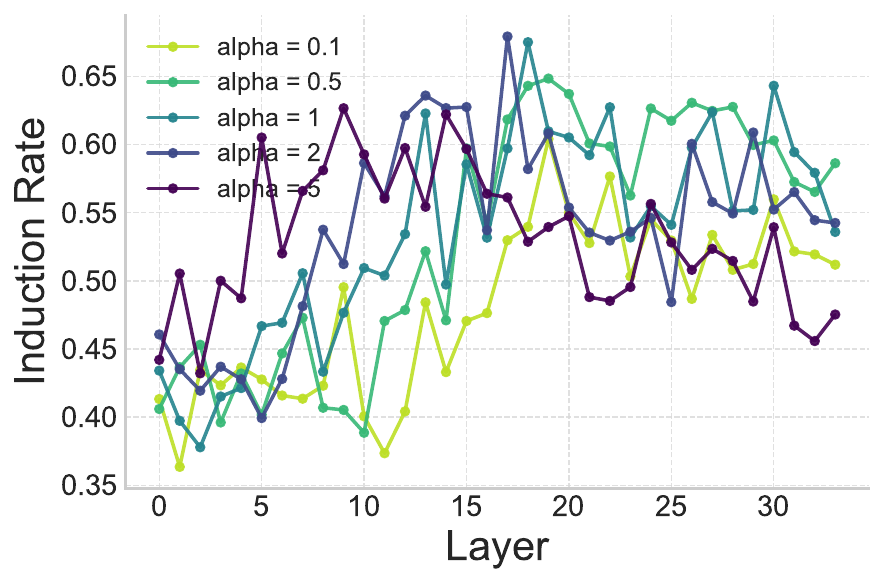} 
    \end{subfigure}
    \begin{subfigure}[b]{0.32\textwidth}
        \centering
        \includegraphics[width=\linewidth]{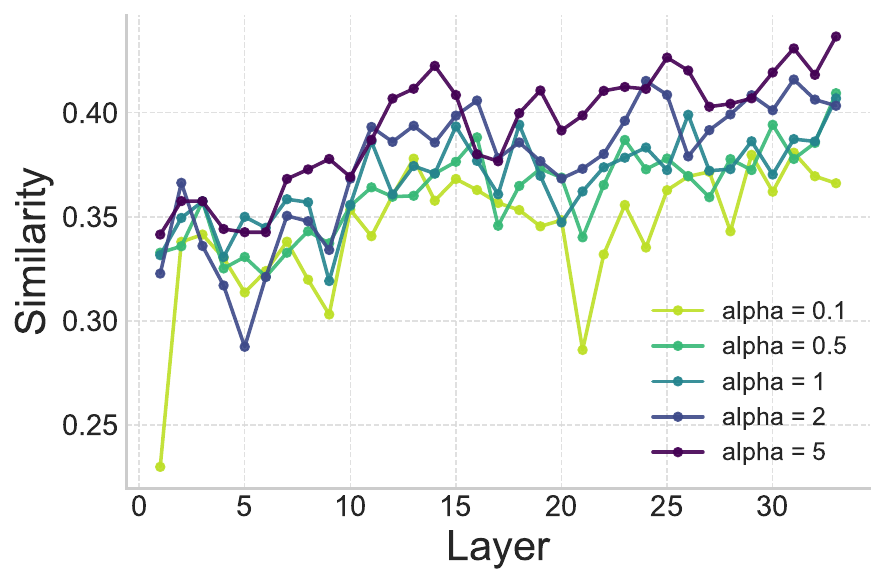}
    \end{subfigure}
    \begin{subfigure}[b]{0.32\textwidth}
        \centering
        \includegraphics[width=\linewidth]{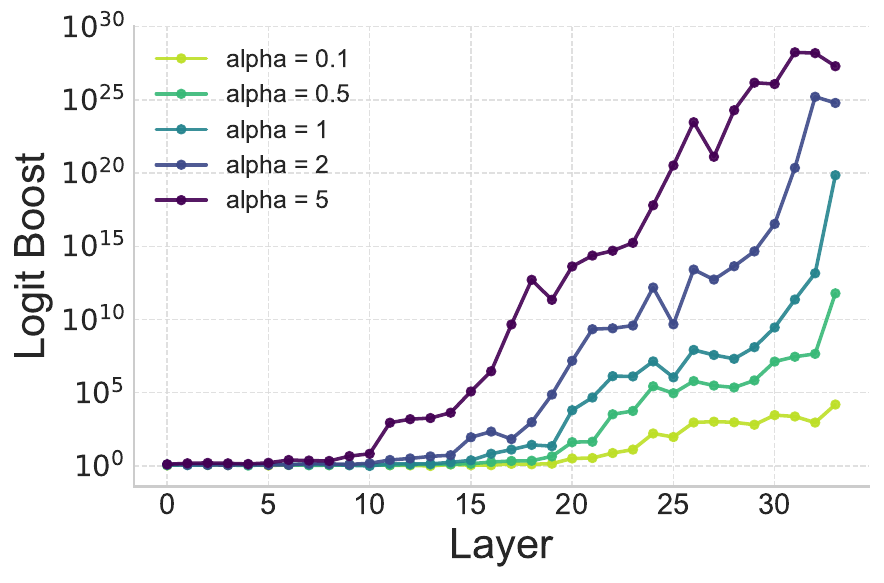}
    \end{subfigure}
    
    \caption{Sensitivity analysis of the steering coefficient $\alpha$. Overall, the metrics demonstrate a positive correlation with both layer and intensity $\alpha$. Similarity and Logit Boost scale monotonically with $\alpha$, indicating that stronger intervention increasingly aligns internal states with the target concept. While a higher $\alpha$ maximizes semantic similarity and logit values, it leads to a decline in the Success Rate at deeper layers due to extremely high logit.}
    \label{fig:alpha}
\end{figure*}

\subsection{Sensitivity Analysis}
To study the influence of intervention strength, we conducted an ablation study on the steering coefficient $\alpha$. This parameter controls the magnitude of the injected concept vector, representing a trade-off between semantic efficacy and the preservation of the model's original activation space.

While moderate intervention enhances control, excessive strength disrupts the generation process. As illustrated in \Cref{fig:alpha}, the Success Rate exhibits a non-linear response to intervention strength. Increasing $\alpha$ from 0.1 to 1 yields significant performance gains and establishes an optimal plateau. Pushing the coefficient to an extreme value (e.g., $\alpha=5$) results in a distinctive inverted-U trajectory. Although extreme intervention achieves high success rates when applied at intermediate layers, the performance collapses significantly when injected into the terminal layers.

Comparing the metrics in \Cref{fig:alpha} reveals a critical divergence between overall semantic alignment and explicit generation. Unlike the Success Rate, Semantic Similarity scales monotonically with $\alpha$, with the $\alpha=5$ setting consistently achieving the highest scores. This contradiction indicates a state of semantic over-saturation: extreme intervention vectors dominate the residual stream, and their disproportionate magnitude overrides the related features necessary for coherent text generation. Consequently, the output space is heavily saturated with the target concept, but the network fails to execute the precise token sequences.

The underlying mechanism of this instability is elucidated by the Logit Boost analysis in \Cref{fig:alpha}. The intervention induces an exponential increase in token logits. While $\alpha=1$ elevates logits to a robust but manageable range, $\alpha=5$ drives them to extreme magnitudes exceeding $10^{25}$. Such \textbf{strong intervention destroys the probability distribution necessary for decoding}, likely trapping the model in repetition loops or incoherent degeneration.

\section{Reverse Steering}
\label{sec:4}
While previous positive steering demonstrated the sufficiency of these vectors in inducing concepts, establishing a robust causal link requires verifying that the generation process causally depends on these specific representations. To this end, we conduct reverse steering, a counterfactual intervention where the target concept vector is subtracted from the residual stream during inference ($h^{l} \leftarrow h^{l} - \alpha v_c^{l}$). By observing the model's response to this targeted ablation, we reveal a fundamental causal asymmetry and a distinct dissociation between internal activation and external behavior. \looseness=-1

\begin{figure}[h!] 
    \centering

    \begin{subfigure}{0.48\linewidth} 
        \centering
        \includegraphics[width=\linewidth]{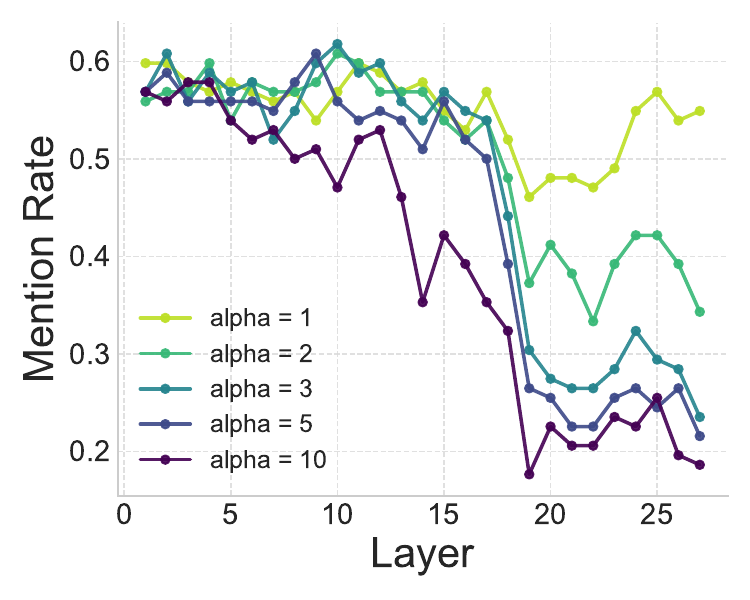} 
    \end{subfigure}
    \begin{subfigure}{0.48\linewidth}
        \centering
        \includegraphics[width=\linewidth]{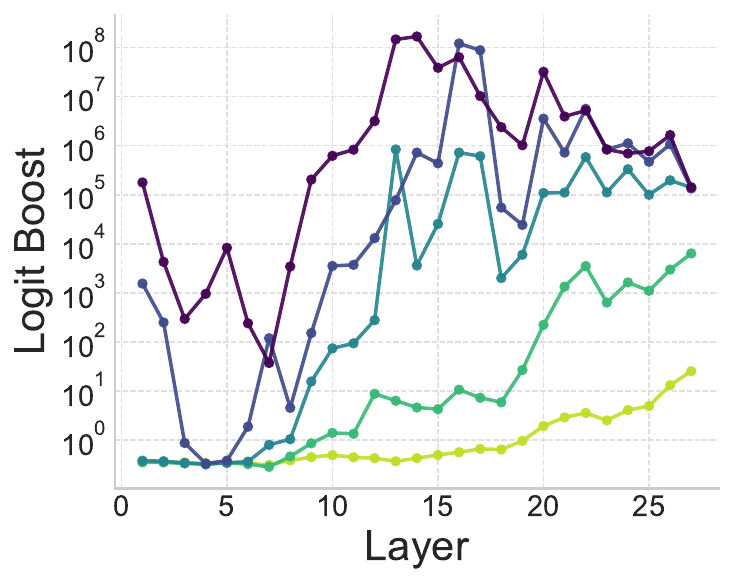}
    \end{subfigure}
    
    \caption{Analysis of reverse steering. \textbf{Left}: The intervention demonstrates high efficacy in suppressing target entities, with the Mention Rate dropping precipitously in deeper layers, confirming the vector's role as a necessary gating factor. \textbf{Right}: Despite the behavioral suppression, the internal Logit Boost for the target token exhibits an anomalous exponential increase. This phenomenon indicates that the model attempts to compensate for suppressed residual information by over-attending to the persistent visual features from image.}
    \label{fig:reverse}
\end{figure}

\paragraph{Behavioral Suppression.}
As our purpose is to suppress the model's representation, we employ Mention Rate, the frequency with which the target entity still appears in the generated response, as the primary behavioral metric. As shown in \Cref{fig:reverse}, the Mention Rate drops significantly during reverse steering, indicating that we successfully severed the link between the visual input and the textual output. Crucially, this suppression experiment establishes necessity: without the specific representation in the residual stream, the network fails to decode the corresponding tokens, even though the raw visual features remain intact in the cross-attention input.

\paragraph{Conflict and Overcompensation.}
Intuitively, ablating a concept vector should diminish its activation levels. However, \Cref{fig:reverse} reveals a stark divergence between the decoding behavior and its internal logit values. While the reverse steering successfully prevents the explicit generation of the target word, the logit values for that specific token exhibit an anomalous surge, reaching magnitudes of $10^6$ to $10^8$ in the deeper layers.

We attribute this extreme latent activation to a feature conflict between the residual stream and the cross-attention mechanism. Although the intervention forcefully removes the target's semantic encoding from the residual stream, the image encoder continuously injects visual evidence via the cross-attention layers. This creates a representational mismatch where the visual stream signals the object's presence despite its absence in the residual stream. Consequently, attention mechanisms likely overcompensate by allocating disproportionate weights to visual patches in a futile attempt to retrieve the missing concept.

\section{Visual Logical Reasoning}
\label{sec:5}
In the previous section, we demonstrated that MLLMs struggle to align abstract concepts with visual features. We want to further explore whether the model can understand visual logical reasoning, a more abstract concept. Specifically, we investigated the internal representation of visual logical reasoning using geometry auxiliary lines \cite{fu2025geolaux} as a case (examples are shown in Appendix \ref{app:geo}). We extracted an ``auxiliary line" concept vector by contrasting original geometry images with corresponding images containing auxiliary lines. To evaluate the intervention effect, we categorize the target words into relationship words (e.g., `parallel', `perpendicular') and action words (e.g., `connect', `extend'). The results are presented in \Cref{fig:logic_alpha}.

\begin{figure}[h!] 
    \centering
    
    \begin{subfigure}{0.48\linewidth} 
        \centering
        \includegraphics[width=\linewidth]{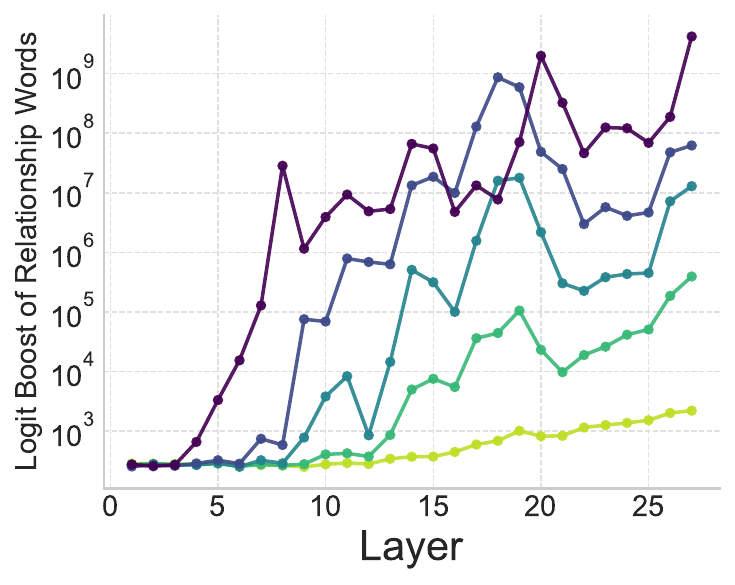} 
    \end{subfigure}
    \begin{subfigure}{0.48\linewidth}
        \centering
        \includegraphics[width=\linewidth]{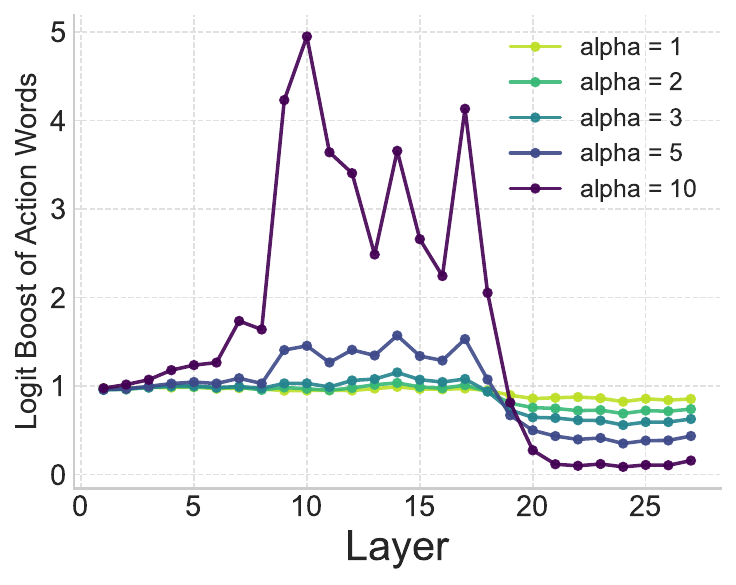}
    \end{subfigure}
    
    \caption{Contrast in Logit Boost between relationship words and action words. The results show that while the steering vector successfully activates descriptions of geometric states, it fails to trigger the corresponding action terms. This suggests that the MLLM perceives auxiliary lines as passive visual signals rather than active logical helpers for solving problems.}
    \label{fig:logic_alpha}
\end{figure}

\textbf{Steering activates perception of geometric meanings but fails to trigger reasoning action.} The results for relationship words follow the same positive trend observed in previous categories, indicating that the model indeed perceives the geometric meanings of the auxiliary lines. It also confirms that the extracted vectors effectively encode the geometric relationships introduced by auxiliary lines. However, action words fail to exhibit a corresponding logit increase and even show a declining trend. This dissociation uncovers a critical gap: although the injection of auxiliary vectors alters the description of geometric meanings, it fails to trigger the action tokens necessary for mathematical reasoning. These results suggest that MLLMs encode auxiliary elements merely as static visual features rather than computational operators that drive reasoning pathways.

To assess the impact of this representational gap on logical reasoning, we extended our evaluation to three different MLLM architectures, measuring the Construction Rate (frequency of explicitly constructing auxiliary) and Answer Accuracy. As shown in \Cref{fig:bench_acc}, the introduction of visual steering vectors resulted in no statistically significant performance improvement across all evaluated models, with accuracy restricted to a narrow margin. For example, the best-performing Gemma3-27B only increased by 1.6\%. The Construction Rate further confirms this disconnect. Even with strong visual cues injected, the models' explicit attempts to construct auxiliary lines remained consistently low, peaking at only 15.3\%. 

\begin{figure}[h!]
    \centering
    \includegraphics[width=\linewidth]{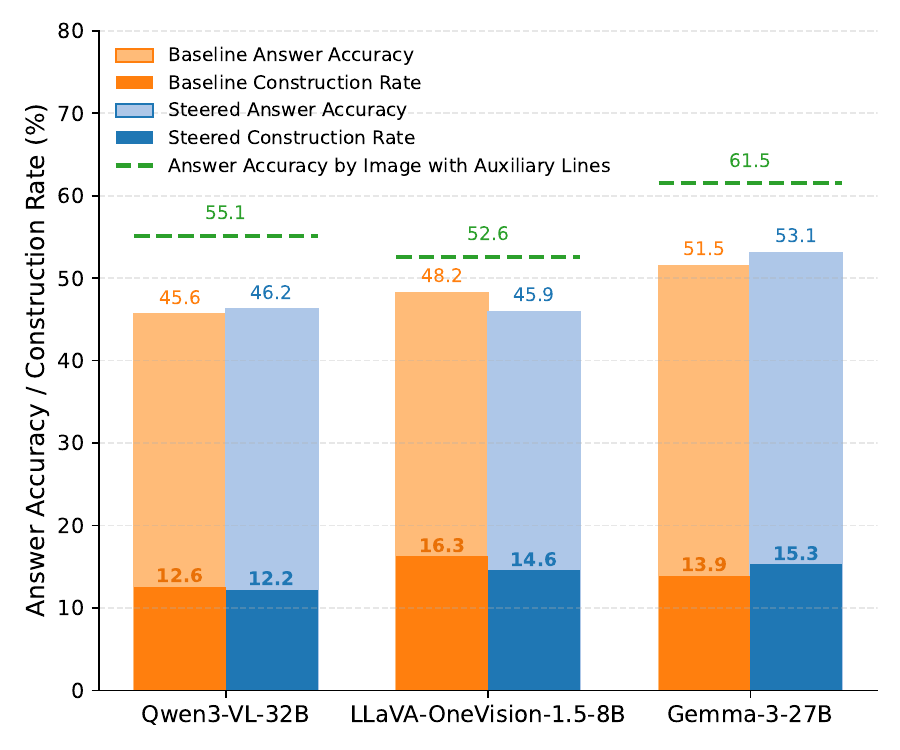}
    \caption{Evaluation of activation steering on geometric reasoning performance across different MLLMs. As shown in the figure, neither metric shows stably significant improvement after visual concept injection. This indicates a fundamental disconnect between visual perception and logical reasoning, where enhanced visual grounding fails to translate into logical strategies.}
    \label{fig:bench_acc}
\end{figure}

This consistent failure across diverse architectures highlights \textbf{a disconnection between visual perception and logical reasoning}. Injected visual vectors, while perceptually salient, remain functionally inert because the model lacks the intrinsic causal mechanism to recognize these visual features as distinct signals for problem-solving interventions. Consequently, while visual steering enhances the explicit geometric description, it fails to build problem-solving trajectories, resulting in stagnant reasoning performance.

\section{Related Work}

\paragraph{Causal Interpretability.} Causal interpretability aims to understand the internal mechanisms by intervening on the model and analyzing its causal effects on specific behaviors \cite{geiger2021causal,li2022emergent,abraham2022cebab,canby2025reliable}.  Early interpretability research relied on attribution \cite{ribeiro2016should,lundberg2017unified}, linear probing \cite{alain2016understanding,tenney2019bert} or attention \cite{jain2019attention,wiegreffe2019attention,yan2026lasm}, but these methods only discover statistical correlations instead of causality \cite{belinkov2022probing}. To bridge this gap, recent advancements include active intervention techniques, such as causal tracing \cite{meng2022locating,meng2022mass,zhang2023towards}, circuit analysis \cite{elhage2021mathematical,wang2022interpretability,conmy2023towards}, and causal mediation analysis \cite{vig2020causal,mueller2024quest}. Extending these, our work introduces an activation steering framework to study how different visual concepts are encoded within MLLMs.

\paragraph{Activation Steering.} 
Also known as representation engineering, it is a lightweight technology to control model behavior by modifying the internal activation value during inference \cite{turner2023steering,bartoszcze2025representation,wehner2025taxonomy}. Based on linear representation hypothesis, this technology has been used to control model personality \cite{chen2025persona}, reduce sycophancy \cite{sharma2023towards,rimsky2024steering} and toxicity \cite{wang2024semantics}, and refuse to answer harmful requests \cite{arditi2024refusal,lee2024programming}; reduce hallucinations and guide model to be more honest and truthful \cite{li2023inference,marks2023geometry,azaria2023internal}; output expected formats \cite{stolfo2024improving}, emotions \cite{turner2023steering,tigges2024language,zhao2024beyond} and styles \cite{azizi2025activation}. Recent research has shifted to MLLMs for solving visual hallucinations \cite{liu2024reducing,sivakumar2025steervlm,gan2025textual,khayatan2025analyzing}. In this work, we consider that visual concepts can also be linearly encoded, and thus we propose a causal framework based on activation steering to explore the mechanism of different visual concepts. \looseness=-1

\section{Conclusion}

In this paper, we propose a causal framework to investigate the internal encoding mechanisms of different visual concepts in MLLMs. Our systematically applied interventions reveal distinct representation patterns across various concepts. Specifically, we highlight divergences in concept encoding, latent compensatory mechanisms, and functional reasoning gaps. Collectively, these findings illuminate the intrinsic mechanisms and capability boundaries of current multimodal models.

\section*{Limitations}
Our steering framework builds upon the linear representation assumption. While this assumption holds well for most concept categories evaluated in our study, it may not perfectly capture highly abstract concepts. Second, our mechanistic interpretations regarding internal dynamics, such as compensatory mechanisms, are inferred from behavioral interventions; validating these requires complementary observational techniques. Finally, our empirical findings are primarily validated on specific mainstream architectures. Future work should conduct systematic evaluations across a broader range of model families to explore universal patterns and cross-architecture generalization.

\section*{Ethical Considerations}
This work advances the field of mechanistic interpretability by enabling precise control over MLLM generation. While this facilitates model understanding and alignment, the ability to induce specific hallucinations or emotional biases introduces risks. For instance, humans with model access could manipulate outputs to deviate from reality. However, our findings regarding the visual verification mechanism suggest that models possess inherent resistance to groundless fabrication. We hope this work inspires further research into internal guardrails that can detect and mitigate such latent space interventions, ensuring MLLMs remain faithful and reliable tools.

\bibliography{custom}

\newpage
\appendix
\section{Mathematical Proof}
\label{app:proof}
In this section, we provide a theoretical justification for using the difference-in-means vector as the steering direction. We demonstrate that this simple arithmetic operation is theoretically optimal under the linear representation hypothesis \cite{belrose2024diff}. Specifically, we first prove that any valid linear concept must align with this direction, and then show that this direction maximizes the intervention effect when the model weights are unknown.

Let $h \in \mathbb{R}^d$ denote the activation vector at a specific layer, and let $z \in \{0, 1\}$ be the binary label representing the presence or absence of a visual concept. We assume the linear representation hypothesis: the concept $z$ is encoded in the activation space via a linear predictor $\eta(h) = w^T h + b$. However, the true parameters $(w, b)$ used by the model are unknown to us.

\begin{definition}
     The trivially attainable loss for labels $z$ and loss $\mathcal{L}$ is the lowest possible expected loss available to a constant predictor $\eta(h)$: $$\mathcal{L}_{\tau} = \inf_{\alpha \in \mathbb{R}} \mathbb{E}[\mathcal{L}(\eta(h), z)]$$
\end{definition}

\begin{definition}
     An admissible predictor for labels $z$ and loss $L$ is a linear predictor whose loss is strictly less than the trivially attainable loss $\mathcal{L}_{\tau}$. 
\end{definition}

\begin{definition}
    A loss function $L(\eta, z)$ is monotonic if it monotonically decreases in $\eta$ when $z=1$, and monotonically increases in $\eta$ when $z=0$. Equivalently, its derivative wrt satisfies:
     $$\mathcal{L}_{\eta}(\eta,1) \leq 0 \leq \mathcal{L}_{\eta}(\eta,0)
     $$
\end{definition}

First, we establish the \textbf{necessity} of the difference-in-means direction. Let $\mu_1 = \mathbb{E}[h \mid z=1]$ and $\mu_0 = \mathbb{E}[h \mid z=0]$ be the class centroids.

\begin{theorem}
    Let $\delta = \mu_1 - \mu_0$ be the difference in class centroids. Suppose $\eta(h) = w^T h + b$ is admissible for $(h,z)$ and convex monotonic $\mathcal{L}$. Then $\langle w, \delta \rangle > 0$.
\end{theorem}

\begin{proof}
    Suppose for the sake of contradiction that $\langle w, \delta \rangle \leq 0$. The expectation of the predictor's output conditioned on the class is: $\mathbb{E}[\eta(h) \mid z] = w^T \mathbb{E}[h \mid z] + b = w^T \mu_z + b$. The difference is:
    $$w^T (\mu_1 - \mu_0) = \langle w, \delta \rangle \leq 0,$$
    and therefore 
    $$\mathbb{E}[\eta(h) \mid z=1] \leq \mathbb{E}[\eta(h)] \leq \mathbb{E}[\eta(h) \mid z=0]$$
    We can now show that the expected loss is lower bounded by the trivially attainable loss $\mathcal{L}$:
$$\begin{array}{r l}
\displaystyle \mathbb{E}[\mathcal{L}(\eta(h), z)] & \displaystyle = \mathbb{E}_z [ \mathbb{E}_{h \mid z} [\mathcal{L}(\eta(h), z)] ] \\[8pt]
& \displaystyle \geq \mathbb{E}_z [ \mathcal{L}(\mathbb{E}_{h \mid z}[\eta(h)], z) ] \\[8pt]
& \displaystyle \geq \mathbb{E}_z [ \mathcal{L}(\mathbb{E}[\eta(h)], z) ]\\[8pt]
& \displaystyle \geq \mathcal{L}_{\tau}
\end{array}$$

    The penultimate step is justified because, by the monotonicity of $\mathcal{L}$ and the inequality derived above, replacing the conditional expectations with the constant $\mathbb{E}[\eta(h)]$ only decreases the loss.

    If $\mathbb{E}[\mathcal{L}(\eta(h), z)] \geq \mathcal{L}_{\tau}$, the classifier cannot be admissible (Def. A.2), contradicting our earlier assumption. Therefore the admissibility of $\eta(h)$ implies $\langle w, \delta \rangle > 0$. \hfill
\end{proof}

Then, we prove that the difference-in-means direction is the \textbf{optimal} intervention strategy in the worst-case scenario. We define an additive intervention as $h' = h + \alpha u$, where $u$ is a unit direction vector. The effectiveness of this steering on a specific model component $w$ is measured by the change in the predictor's output: $\Delta \eta = \eta(h') - \eta(h) = \alpha (w^T u)$. When $w$ is unknown, we can successfully perform additive edits by selecting $u$ to maximize the worst-case directional derivative.

\begin{theorem}
    Let $\mathcal{H}$ denote the set of all admissible predictors. Then the maximin directional derivative objective
    $$u^* = \underset{\|u\|=1}{\arg\max} \inf_{\eta \in \mathcal{H}} (w^T u)$$
    is maximized by the difference-in-means direction $u^* = \frac{\delta}{\|\delta\|}.$
\end{theorem}

\begin{proof}
    From Theorem A.4, we know that for any $\eta \in \mathcal{H}$, the weight vector $w$ lies in the open half-space defined by $\delta$: $\{w \mid w^T \delta > 0\}$. We can decompose any $w$ into a component parallel to $\delta$ and a component orthogonal to it:
    $$w = c \delta + w_{\perp}, \quad \text{where } c > 0 \text{ and } w_{\perp} \perp \delta.$$
    Consider the directional derivative $w^T u$. If we choose any direction $u$ that is not parallel to $\delta$, it will have a non-zero projection onto the orthogonal subspace $\delta^{\perp}$. Since the set of admissible predictors $\mathcal{H}$ is unconstrained in the orthogonal directions, an adversary can choose a $w_{\perp}$ such that $w_{\perp}^T u$ is arbitrarily large and negative, thereby minimizing $w^T u$. However, if we choose $u$ to be parallel to $\delta$ (i.e., $u = \frac{\delta}{\|\delta\|}$), then: $$w^T u = (c \delta + w_{\perp})^T \frac{\delta}{\|\delta\|} = c \|\delta\| + 0 = c \|\delta\|.$$
    Since $c > 0$ (Theorem A.4) and $\|\delta\| > 0$, this value is strictly positive. Thus, the difference-in-means direction is the only direction that guarantees a positive intervention effect for all admissible linear encodings of the concept. Any other direction risks being orthogonal to or opposing the model's specific implementation of the concept. \hfill
\end{proof}

These theorems mathematically validate that extracting the vector $v_c$ via simple subtraction is not merely a heuristic, but the theoretically optimal strategy for causal intervention when the model's true parameters are unknown.

\section{Data}
\label{app:data}

\subsection{Entity}
We select the instruction-guided image editing dataset Imgedit \cite{ye2025imgedit} and MagicBrush \cite{Zhang2023MagicBrush} as the source data. Through python script, we filter out the image pairs containing "remove" in the editing instructions to obtain the image pairs with and without target entities. Subsequently, we also use LLM to further filter the instructions to ensure that only the object is simply deleted in the instructions, rather than replaced by another object. At the same time, record the name of the entity as the ground truth label. Prompt is as follows:
\begin{tcolorbox}[title=\textbf{Prompt}, colback=gray!10, colframe=gray!50!black, arc=1mm]
\small 
You are a helpful assistant for image-editing dataset processing. Your task is to extract the object being removed from the user's instruction. \\
Rules: \\
1. If the instruction is purely about removing objects, extract the object name. \\
2. If the instruction contains other actions like "add", "place", "change", "replace", or implies adding something new (e.g., "remove X and add Y"), output exactly "SKIP". \\
3. Do not output a full sentence. Output ONLY the object name or "SKIP". \\
Examples: \\
Input: "Remove the candles." Output: candle \\
Input: "Remove the horses that are on the bridge." Output: horse \\
Input: "Remove the vegetables and add candies to the bowl." Output: SKIP \\
Input: "Remove all the food and place a laptop on the table." Output: SKIP
\end{tcolorbox}
We finally construct 10,000 image pairs, and entity names are shown in \Cref{fig:app_cloud}.

\begin{figure}[h!]
    \centering
    \includegraphics[width=\linewidth]{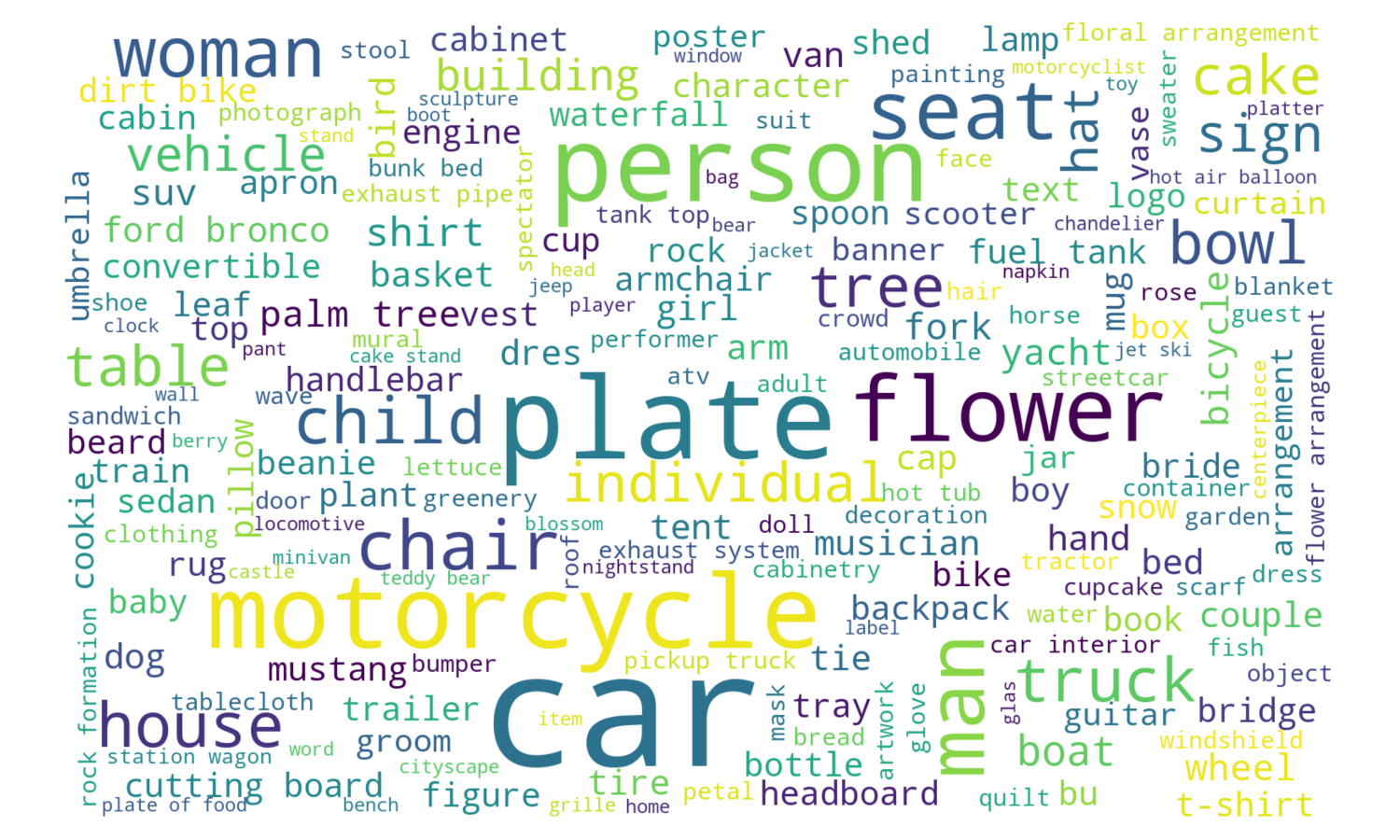}
    \caption{Word Cloud of entities.}
    \label{fig:app_cloud}
\end{figure}

\subsection{Visual Style} 
To construct the concept vectors for visual style, we utilize the OmniConsistency dataset \cite{song2025omniconsistency}, a high-quality paired stylization benchmark. This dataset contains 2,600 paired images covering 22 distinct visual styles, including American Cartoon, Oil Painting, Sketch, Pixel Art, and Vector Style. The dataset was constructed using a GPT-4o-driven generation pipeline followed by a rigorous human-in-the-loop filtering process. Specifically, for each style, the dataset provides an original image (typically photorealistic or neutral) and a corresponding stylized version. The strict filtering criteria ensured that the stylized images preserve the semantic structure, identity, and layout of the original images while manifesting strong stylistic attributes. This high degree of content alignment is critical for our activation steering method, as it ensures that the difference vector captures the style direction rather than object semantics.

\subsection{Emotion} 
For the Emotion category, we utilize EmoSet \cite{yang2023emoset}, a large-scale visual emotion dataset annotated based on Mikels’ eight-emotion model. The dataset comprises 3.3 million images with rich attributes, covering eight categories: Amusement, Awe, Contentment, Excitement, Anger, Disgust, Fear, and Sadness. 

Unlike entity or style, where paired data is explicitly available, emotional concepts are often entangled with scene content. To isolate the active emotional component, we adopt a baseline-subtraction strategy. In the psychological emotion space, Contentment represents a state of low physiological arousal and psychological equilibrium, often associated with stable, peaceful scenes (e.g., landscapes, static objects) \cite{yang2023emoset}. So, we designate Contentment as the negative set for all other emotion concepts and extracted 7 concept vectors.

\subsection{Abstract Concept}
Constructing a dataset for abstract concepts (e.g., peace, justice, danger) presents a unique challenge, as these concepts lack a singular, fixed physical form and are often conveyed through symbolic imagery or high-level scene semantics. 

We curated the dataset by crawling high-quality photography from Pexels. This platform was selected because its content is frequently tagged with abstract conceptual keywords by human creators, ensuring a strong correlation between the visual symbolism and the semantic label. We constructed a dataset covering 20 distinct abstract concepts. 

To extract a robust direction for abstract concepts, we adopt antonyms for selecting negative samples, which are recorded in \Cref{tab:app_abs}. Taking the concept of ``peace" as an instance: The positive set consists of imagery symbolizing harmony, such as doves, tranquil landscapes, or shaking hands. The negative set consists of imagery symbolizing conflict, such as battlefield scenes, destruction, or aggressive confrontations. By computing the difference-in-means between these opposed sets, the resulting concept vector effectively captures the semantic axis from conflict to harmony, filtering out unrelated visual features.

\begin{table*}[h!]
    \centering
    \resizebox{\textwidth}{!}{%
    \begin{tabular}{ll|ll|ll|ll}
        \toprule
        \textbf{keywords} & \textbf{antonyms} & \textbf{keywords} & \textbf{antonyms} & \textbf{keywords} & \textbf{antonyms} & \textbf{keywords} & \textbf{antonyms} \\
        \midrule
        peace       & war           & ancient     & modern     & danger        & safe        & freedom     & captivity \\
        hope        & despair       & luxury     & poverty       & stress     & relaxation     & wisdom      & stupid \\
        shame       & honor      & justice      & injustice       & creativity  & imitation     & chaos       & order \\
        humility        & arrogance          & strength    & weakness      & evil       & goodness           & profane      & sacred \\
        failure     & success        & reality     & fantasy       & knowledge   & ignorance     & complexity  & simplicity \\
        \bottomrule
    \end{tabular}
    }
    \caption{List of keywords and antonyms for abstract concept during data collection. We use the keyword to retrieve positive samples and the antonym for negative samples.}
    \label{tab:app_abs}
\end{table*}

\subsection{Logical Reasoning}
\label{app:geo}
For the investigation of visual logical reasoning in Section 6, we use the GeoLaux dataset \cite{fu2025geolaux}, a specialized dataset designed to evaluate MLLMs on geometry problems that require auxiliary line construction. GeoLaux is ideal because it provides both the initial problem state and the solution state in the form of geometric constructions.

We use the original problem images as negative sets, which contain only the initial geometric shapes without any auxiliary constructions. And we use the corresponding solution images (annotated images) as positive sets, where the key auxiliary lines (e.g., connecting lines, extensions, parallel lines) are explicitly drawn. Examples are shown in \Cref{fig:app_geo}.

\begin{figure*}[h!]
    \centering
    \includegraphics[width=\linewidth]{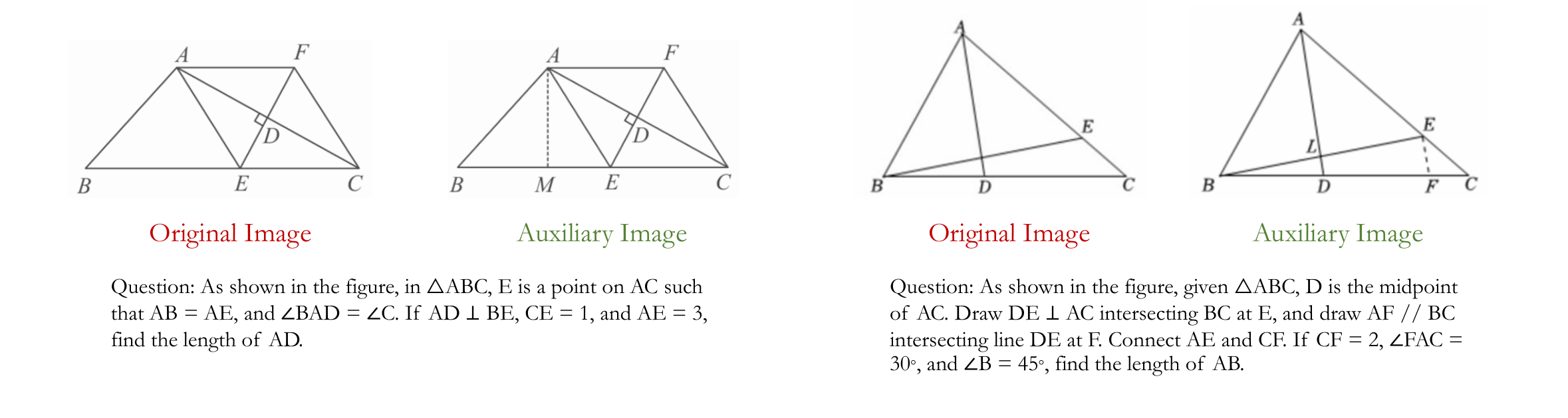}
    \caption{Examples of GeoLaux dataset.}
    \label{fig:app_geo}
\end{figure*}

\section{Experiment Details}

\subsection{Metrics}
\label{app:metric}
We utilize three metrics to rigorously evaluate the causal influence of visual concepts from different perspectives.

\paragraph{Success Rate.} This metric acts as a proxy for causal steerability. It serves not merely to count keyword occurrences but to measure the intervention's ability to force the model to verbalize a specific visual concept it might otherwise ignore. Specifically, we define the Success Rate as the proportion of generated responses where the target concept is explicitly identified by an gpt-4o. A high Success Rate indicates that the internal steering vector successfully overrides the default generation trajectory to trigger the target semantic output. The prompt is as follows.

\begin{tcolorbox}[title=\textbf{Prompt}, colback=gray!10, colframe=gray!50!black, arc=1mm]
\small 
You are an objective evaluator for text generation. Your task is to determine if a specific style or concept was successfully injected into the text.

I will provide you with a Target Concept, a Base Text (before injection), and a Steered Text (after injection).

Target Concept: \{target\_word\}

Base Text:\{base\_output\}

Steered Text:\{steered\_output\}

Evaluation Criteria:

1. Does the Steered Text explicitly mention the Target Concept or words highly related to it?

2. Does the Steered Text describe the image in the style of the Target Concept compared to the Base Text?

3. If the Base Text already contained the concept, answer 1 only if the Steered Text maintains or emphasizes it.

Task:

Output '1' if the injection is SUCCESSFUL (the concept is clearly present).

Output '0' if the injection is FAILED (the concept is missing or the text is nonsense).

Output Format:

Only output a single digit: 0 or 1. Do not output any other text.
\end{tcolorbox}

\paragraph{Sentence-BERT Similarity.}
While Success Rate captures explicit discrete mentions, this metric captures dense semantic proximity. We compute the cosine similarity between the embeddings of the generated text and the target concept description using Sentence-BERT \cite{reimers2019sentence}, \texttt{all-mpnet-base-v2} version. This allows us to assess whether the steering intervention successfully shifts the global context and style of the text, even if the specific target word is not explicitly mentioned.

\paragraph{Logit Boost.}
To verify that behavioral changes stem from internal representation shifts rather than random decoding noise, we define Logit Boost in its exponentiated form: $e^{Logit_{steered}} / e^{Logit_{base}}$. This formulation quantifies the multiplicative factor applied to the target token's unnormalized probability immediately after the steering intervention, directly reflecting the exponential strength of the injected signal within the residual stream.

\subsection{Prompt for Concept Extraction}
\label{app:prompt}
To ensure that the extracted concept vectors precisely capture the target semantic dimension and minimize the interference of entangled features, we employed specific prompts for different categories during the forward pass of the concept extraction phase, which is shown at \Cref{tab:app_prompt}.

\begin{table*}[h!]
    \centering
    \begin{tabular}{ll}
        \toprule
        \textbf{Category} & \textbf{Prompt} \\
        \midrule
        \textbf{Entity} & 
        Describe the image concisely. \\
        \addlinespace 
        
        \textbf{Visual Style} & 
        Describe the style of the image. \\
        \addlinespace
        
        \textbf{Emotion} & 
        Describe the emotional atmosphere of the image.\\
        \addlinespace
        
        \textbf{Abstract Concept} & 
        Describe the deep meaning of the image.\\
        \addlinespace

        \textbf{Logical Reasoning} & \textit{Instance-specific question stem}
        \\
        \bottomrule
    \end{tabular}
    \caption{List of category-specific prompts used during concept extraction.}
    \label{tab:app_prompt}
\end{table*}

By using directed prompts, we explicitly guide the model's internal representation to the relevant subspace before computing the activation difference.

\section{Faithfulness}
\label{app:Faithfulness}
To rigorously verify whether the linear vectors truly represent the target concepts rather than merely acting as arbitrary adversarial directions, we introduce a Faithfulness metric $\rho$. This metric quantifies the extent to which our causal intervention replicates the model's natural behavior when perceiving the actual visual concept.

We decompose the visual difference's impact on the model's output logits into two components: Total Effect and Natural Indirect Effect.

\textbf{Total Effect (TE)}. This represents the ground-truth change in logits caused by the actual presence of the visual concept in the input image. For a target concept token $y_c$, the TE is calculated as the logit difference between the positive image $x_{c,i}^{+}$ and the negative image $x_{c,i}^{-}$:
\begin{align}
    TE &= \sum_{i=1}^{N} \bigg( \text{logit}(y_c \mid x_{c,i}^{+}, x^t) \notag \\
    &- \text{logit}(y_c \mid x_{c,i}^{-}, x^t) \bigg).
\end{align}

\textbf{Natural Indirect Effect (NIE)}. This represents the logit change induced solely by our activation steering intervention. Here, we input the negative samples (which lack the concept) but artificially inject the extracted vector $v_c^{l}$ to simulate the concept's presence:
\begin{align}
NIE &= \sum_{i=1}^{N} \bigg( \text{logit}(y_c \mid x_{c,i}^{-}, x^t, h^{l} \leftarrow h^{l} + \alpha v_c^{l}) \notag \\ 
&\quad - \text{logit}(y_c \mid x_{c,i}^{-}, x^t) \bigg).
\end{align}

\textbf{Faithfulness} ($\rho$). Based on these definitions, we define the faithfulness ratio $\rho$ as:
\begin{equation}
    \rho = \frac{NIE}{TE}, 
\end{equation}
A $\rho$ value approaching 1 indicates that the intervention vector $v_c^{l}$ perfectly faithfully reconstructs the causal mechanism triggered by the actual visual stimulus. It implies the vector has captured the precise "direction" the model uses to encode that concept. A low $\rho$ suggests the vector might be orthogonal to the true reading mechanism, or that the concept relies on non-linear computations that linear steering cannot capture.

\Cref{tab:faithfulness} presents the faithfulness metrics $\rho$ across six models. We observe a consistently high degree of faithfulness, with scores exceeding 0.85 across most categories. This provides strong empirical evidence that our extracted steering vectors are not merely adversarial perturbations, but genuine representations of the underlying visual concepts. Notably, a clear scaling law emerges from the results: larger models, such as Qwen3-VL-32B and Gemma-3-27B, exhibit superior faithfulness compared to their smaller counterparts, particularly in encoding complex Abstract Concepts (reaching 0.88). This trend suggests that as model capacity increases, the internal representation of visual concepts becomes more linear and disentangled, thereby enhancing the precision of activation steering. Furthermore, the method demonstrates remarkable robustness across distinct architectures and model families, confirming its universality as a mechanism-agnostic interpretability tool.

\begin{table}[t]
    \centering
    \resizebox{\linewidth}{!}{%
    \begin{tabular}{l|cccc}
        \toprule
        \textbf{Model} & \textbf{Entity} & \textbf{Style} & \textbf{Emotion} & \textbf{Abstract} \\
        \midrule
        Qwen2.5-VL-7B       & 0.88 & 0.92 & 0.85 & 0.78 \\
        Qwen3-VL-8B         & 0.92 & 0.94 & 0.89 & 0.82\\
        Qwen3-VL-32B        & 0.95 & 0.97 & 0.93 & 0.88 \\
        LLaVA-OneVision 8B  & 0.93 & 0.90 & 0.86 & 0.80 \\
        Gemma-3-4B          & 0.92 & 0.85 & 0.93 & 0.84 \\
        Gemma-3-27B         & 0.98 & 0.93 & 0.94 & 0.84 \\

        \bottomrule
    \end{tabular}
    }
    \caption{Faithfulness Analysis across Different Models.}
    \label{tab:faithfulness}
\end{table}

\section{More Results}

\subsection{Layer-wise Distribution of Different Architectures}
\label{app:arc}

\Cref{fig:app_qwen3}, \ref{fig:app_qwen2.5}, \ref{fig:app_llava} show different layer-wise distribution of different models. And Gemma3 family's results are shown in \Cref{fig:distribution}. It shows intra-family similarity alongside inter-family divergence.

\begin{figure*}[h!]
    \centering
    \begin{subfigure}[b]{0.24\textwidth} 
        \centering
        \includegraphics[width=\linewidth]{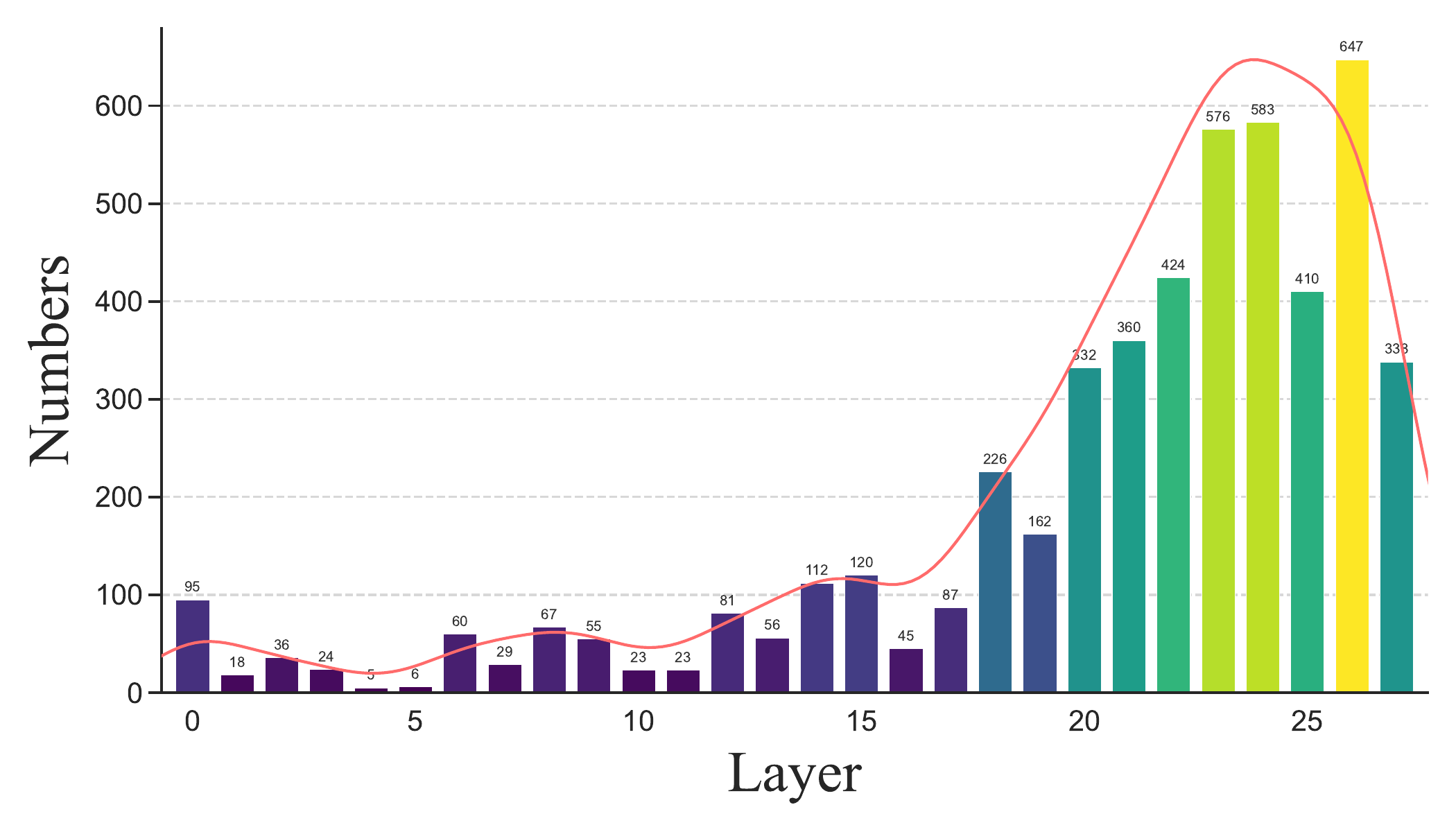} 
    \end{subfigure}
    \begin{subfigure}[b]{0.24\textwidth}
        \centering
        \includegraphics[width=\linewidth]{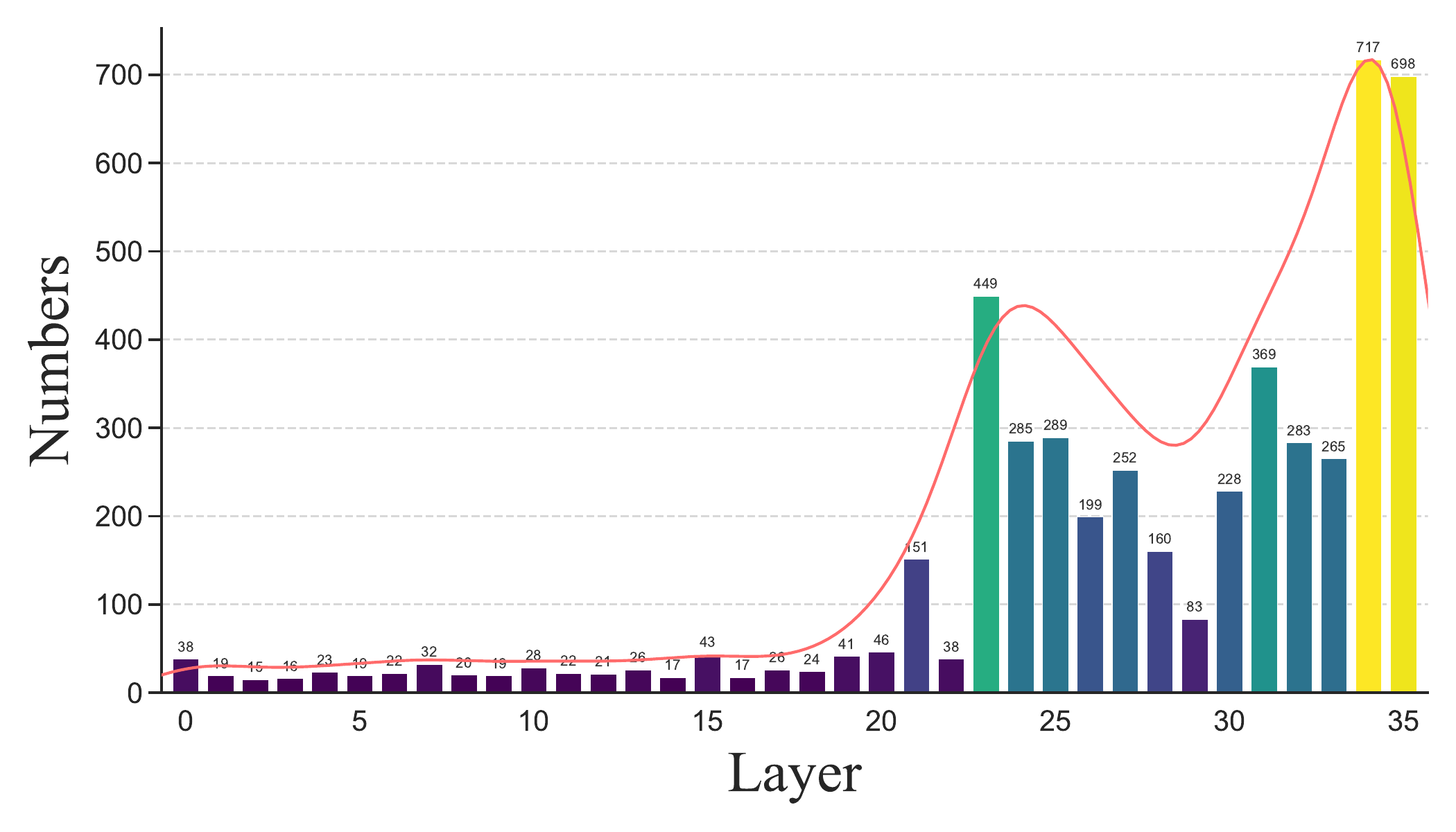}
    \end{subfigure}
    \begin{subfigure}[b]{0.24\textwidth}
        \centering
        \includegraphics[width=\linewidth]{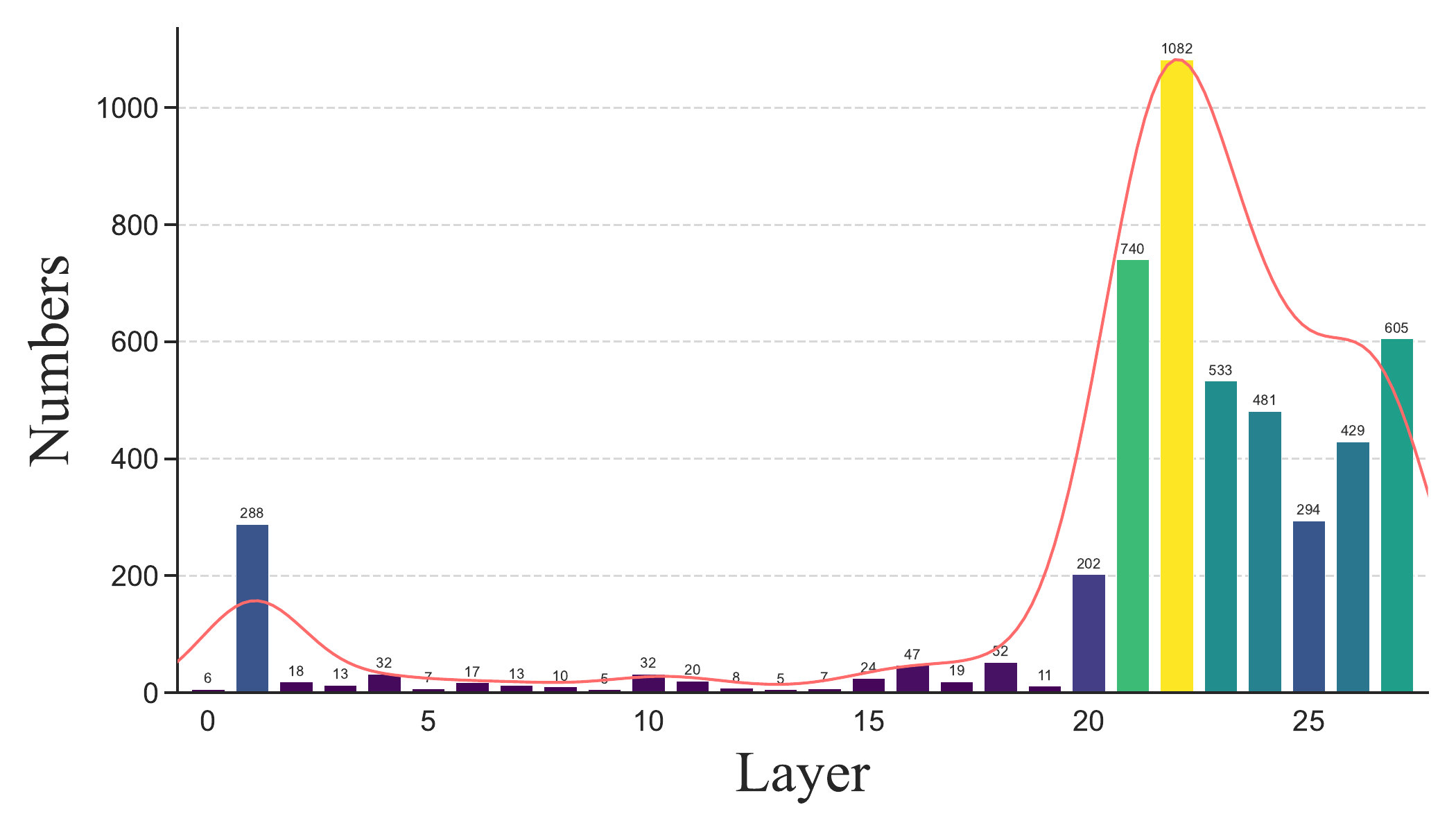}
    \end{subfigure}
    \begin{subfigure}[b]{0.24\textwidth}
        \centering
        \includegraphics[width=\linewidth]{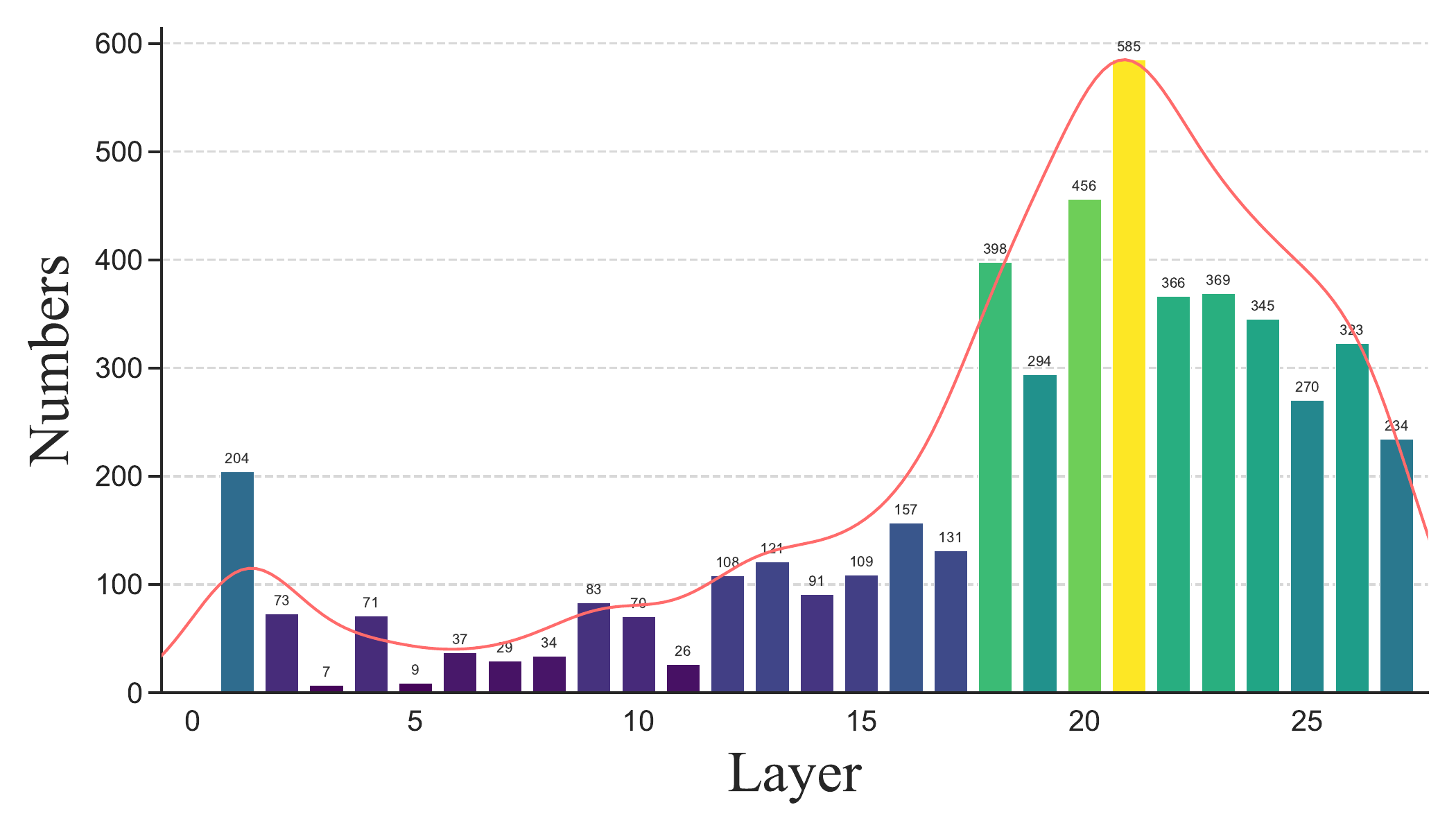}
    \end{subfigure}
    
    \vspace{0.5em} 
    
    \begin{subfigure}[b]{0.24\textwidth}
        \centering
        \includegraphics[width=\linewidth]{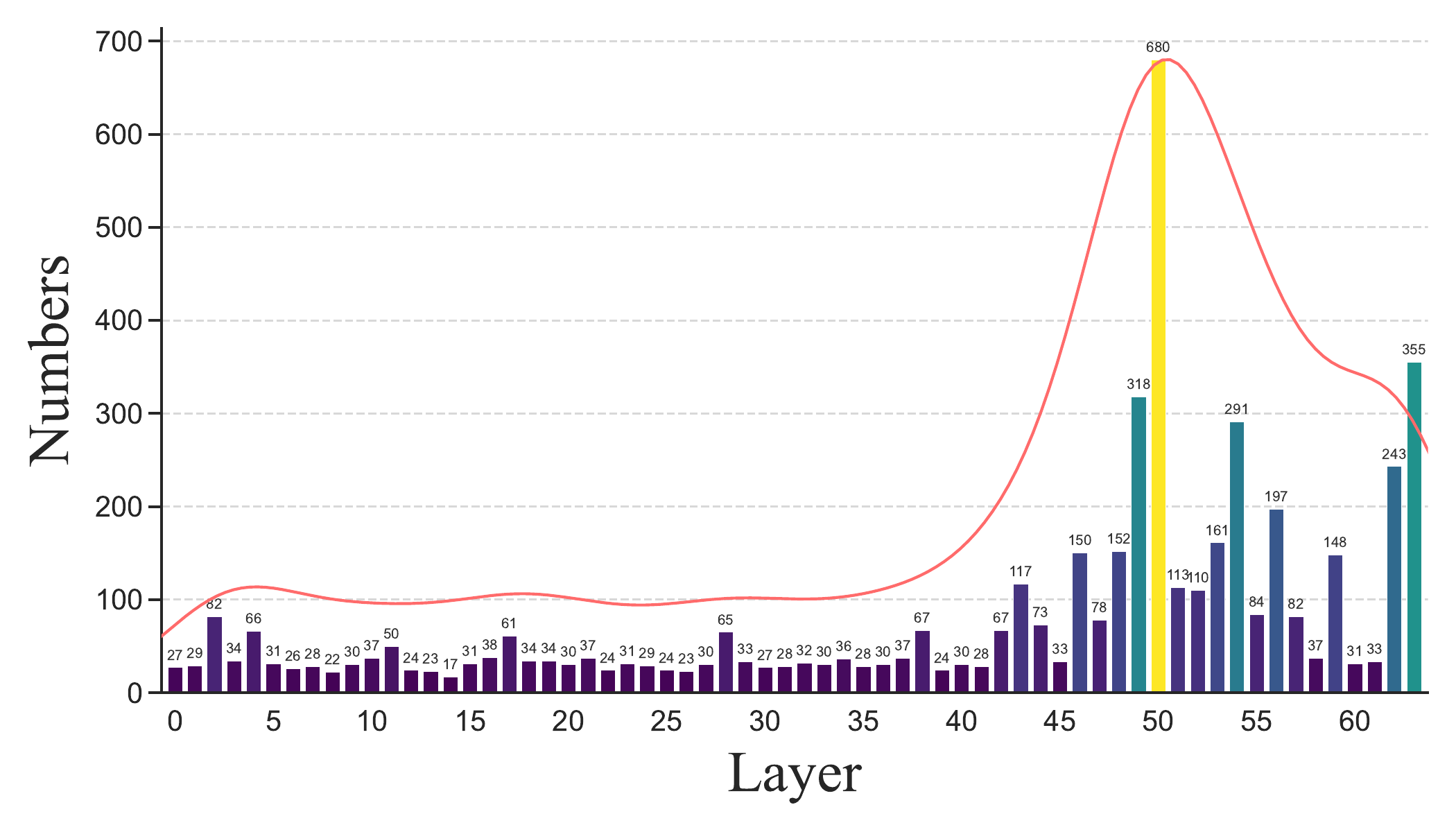}
    \end{subfigure}
    \begin{subfigure}[b]{0.24\textwidth}
        \centering
        \includegraphics[width=\linewidth]{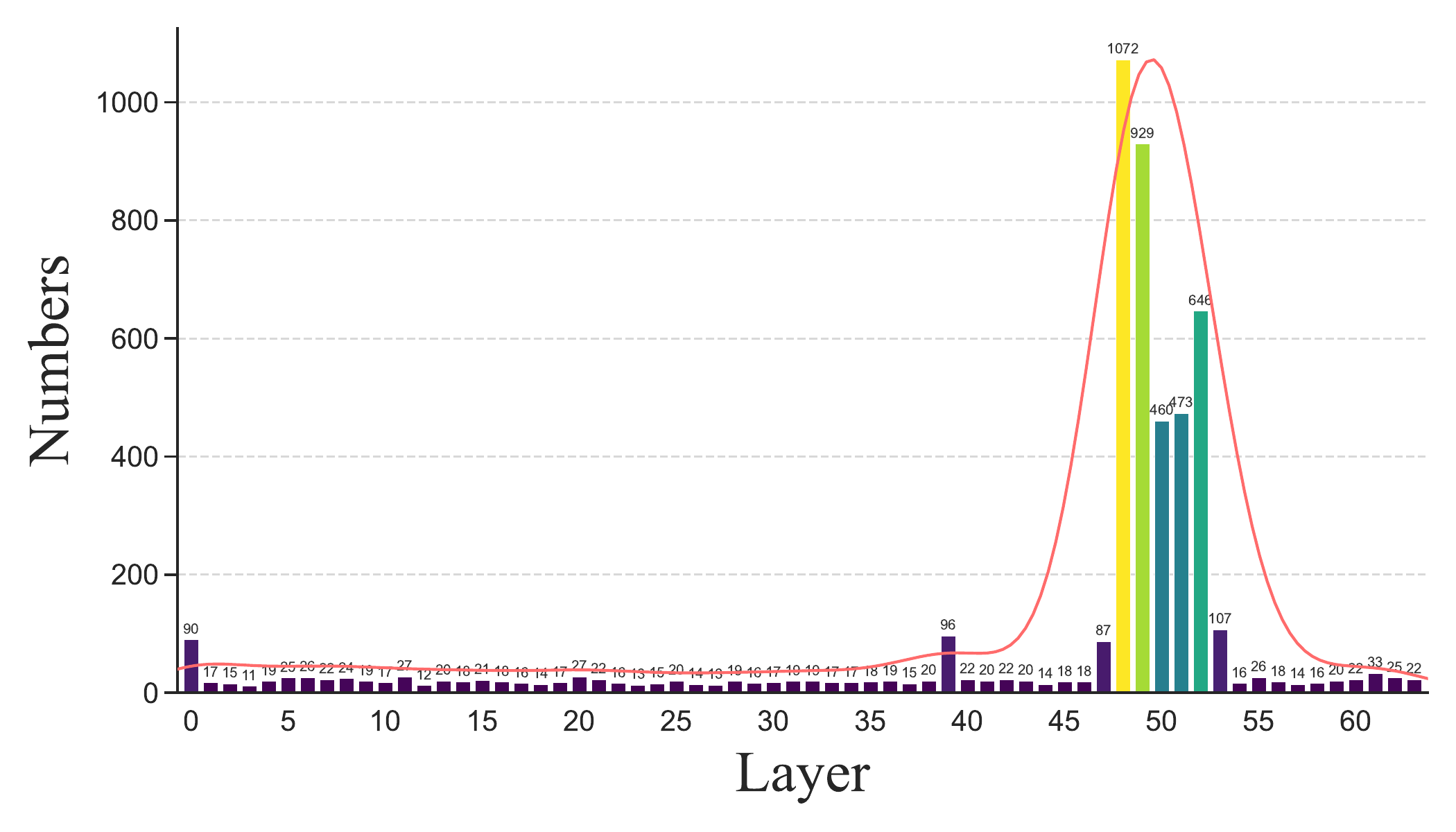}
    \end{subfigure}
    \begin{subfigure}[b]{0.24\textwidth}
        \centering
        \includegraphics[width=\linewidth]{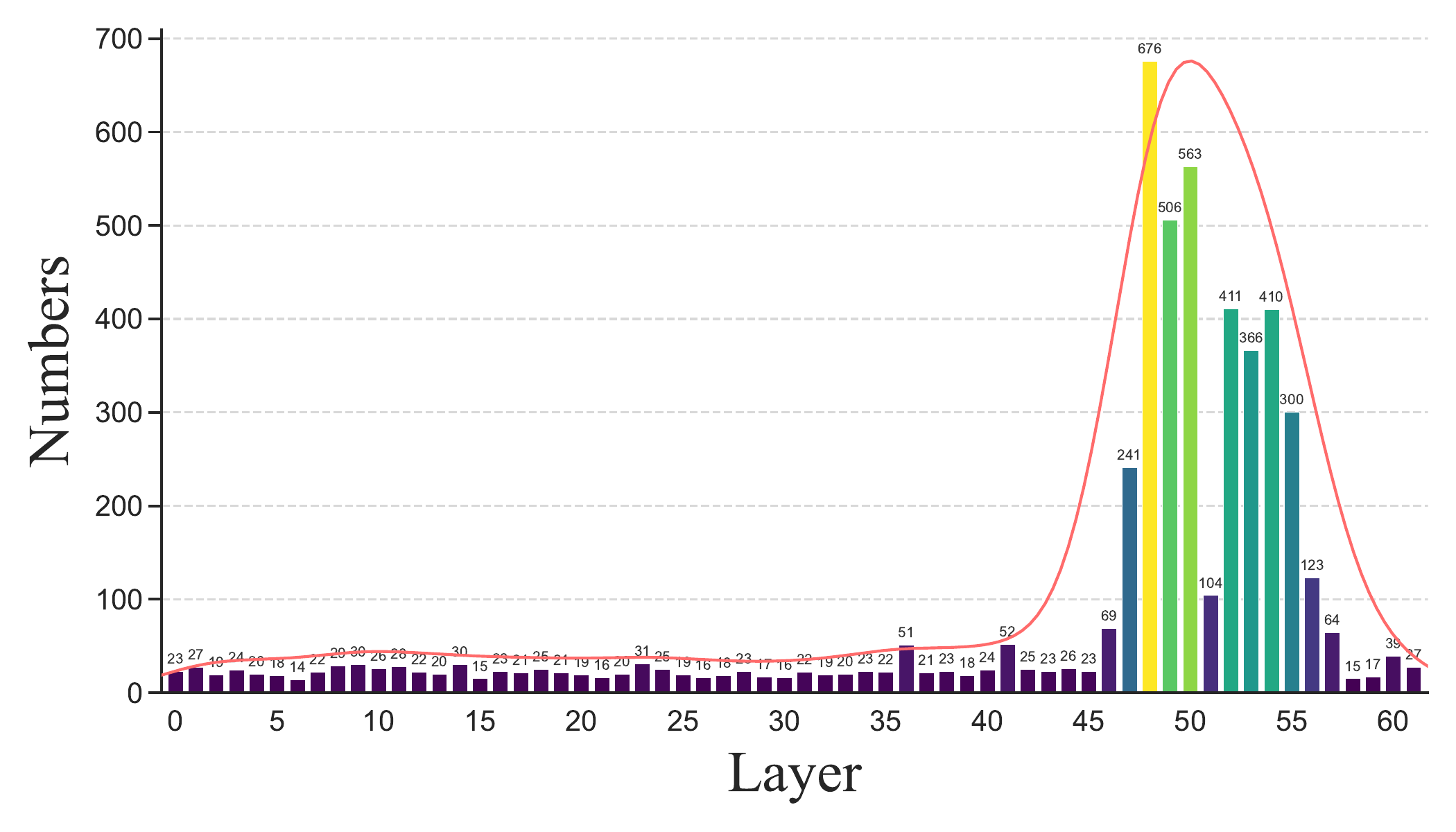}
    \end{subfigure}
    \begin{subfigure}[b]{0.24\textwidth}
        \centering
        \includegraphics[width=\linewidth]{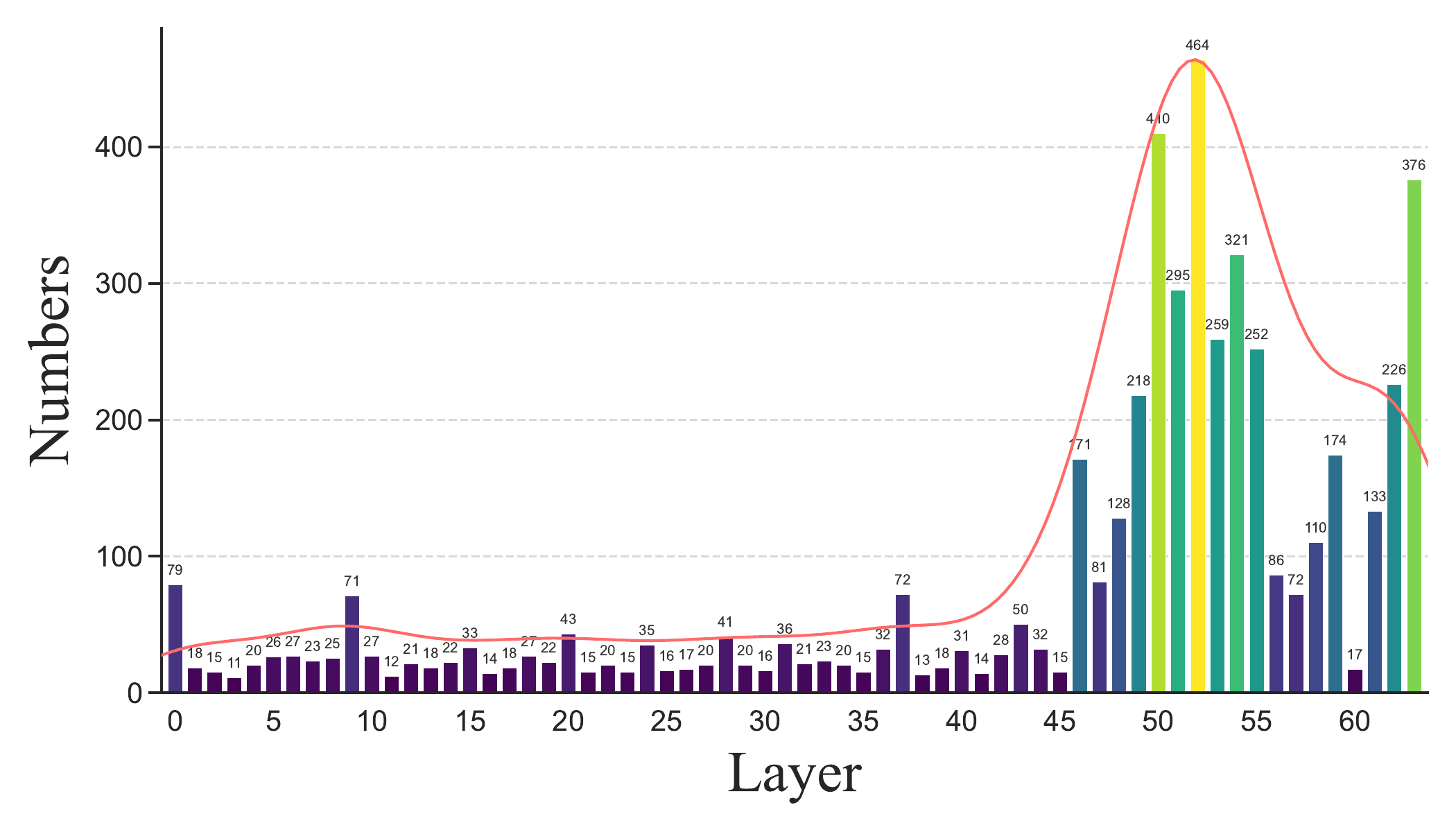}
    \end{subfigure}
    
    \caption{Layer-wise distribution of optimal steering layers of Qwen3-VL. The plots are organized by concept category (columns, left to right: entity, visual style, emotion, and abstract concept) and model scale (rows: top for Qwen3-VL-8B, bottom for Qwen3-VL-32B).}
    \label{fig:app_qwen3}
\end{figure*}

\begin{figure*}[h!]
    \centering

    \begin{subfigure}[b]{0.24\textwidth}
        \centering
        \includegraphics[width=\linewidth]{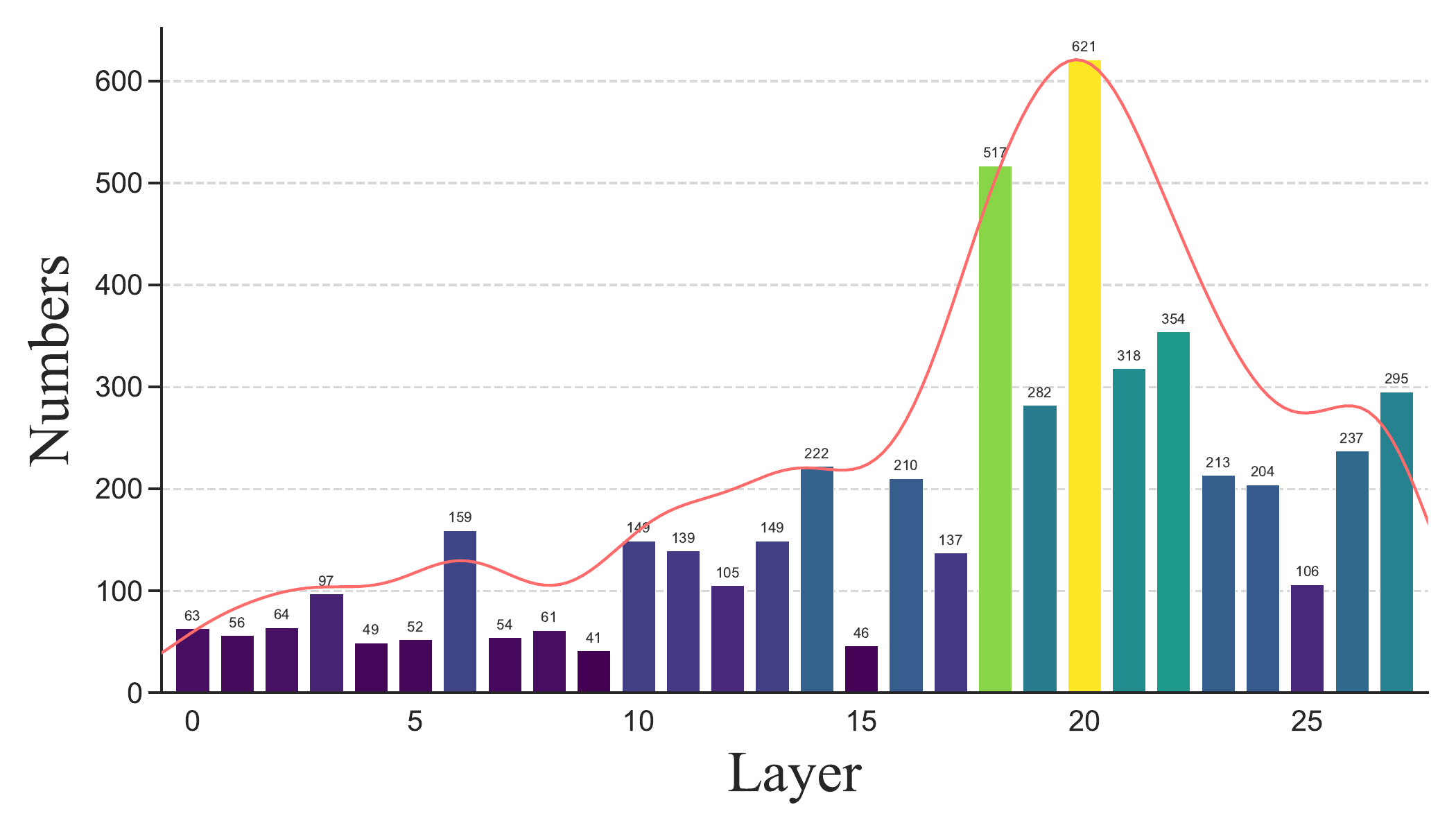}
    \end{subfigure}
    \begin{subfigure}[b]{0.24\textwidth}
        \centering
        \includegraphics[width=\linewidth]{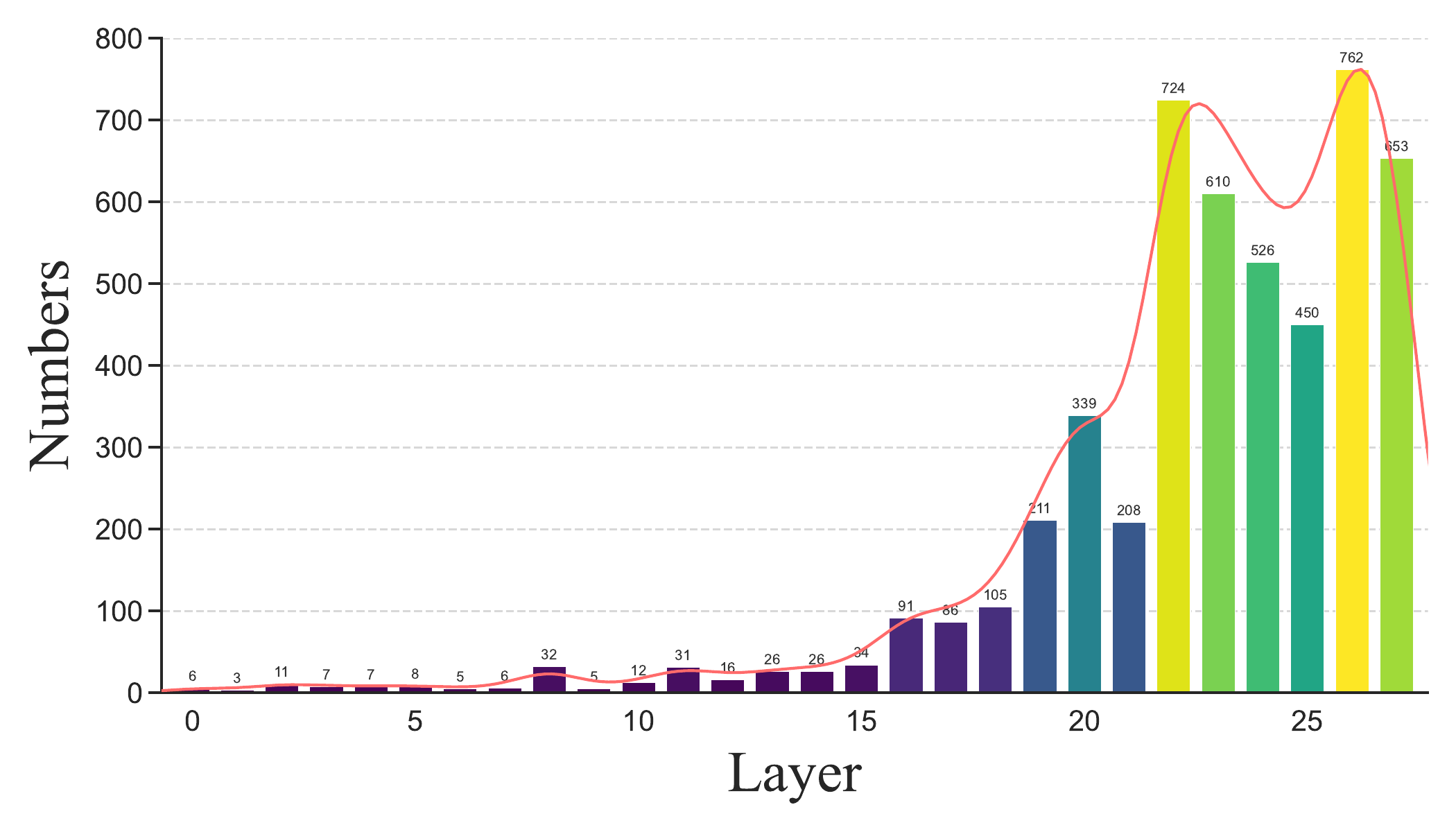}
    \end{subfigure}
    \begin{subfigure}[b]{0.24\textwidth}
        \centering
        \includegraphics[width=\linewidth]{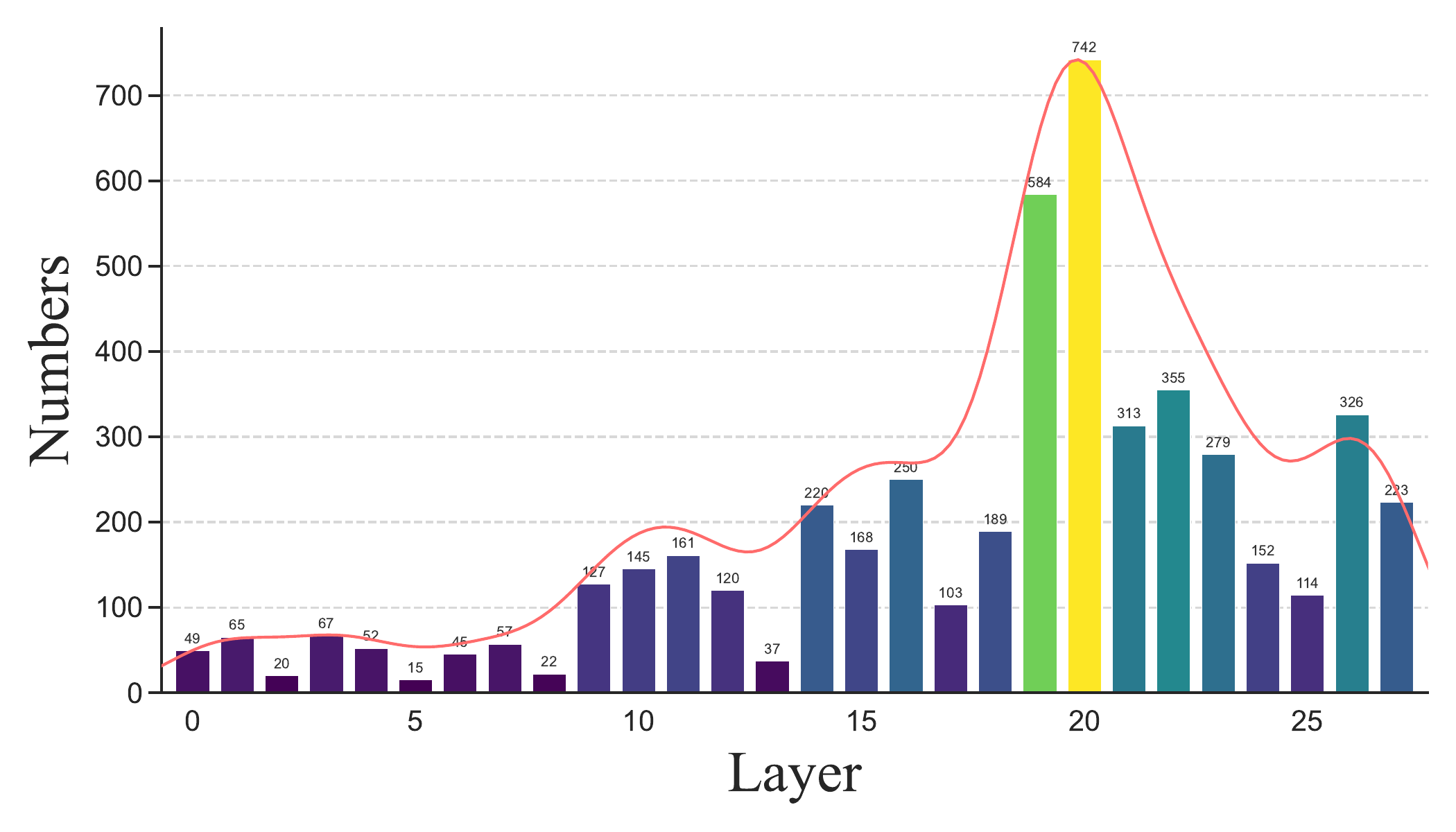}
    \end{subfigure}
    \begin{subfigure}[b]{0.24\textwidth}
        \centering
        \includegraphics[width=\linewidth]{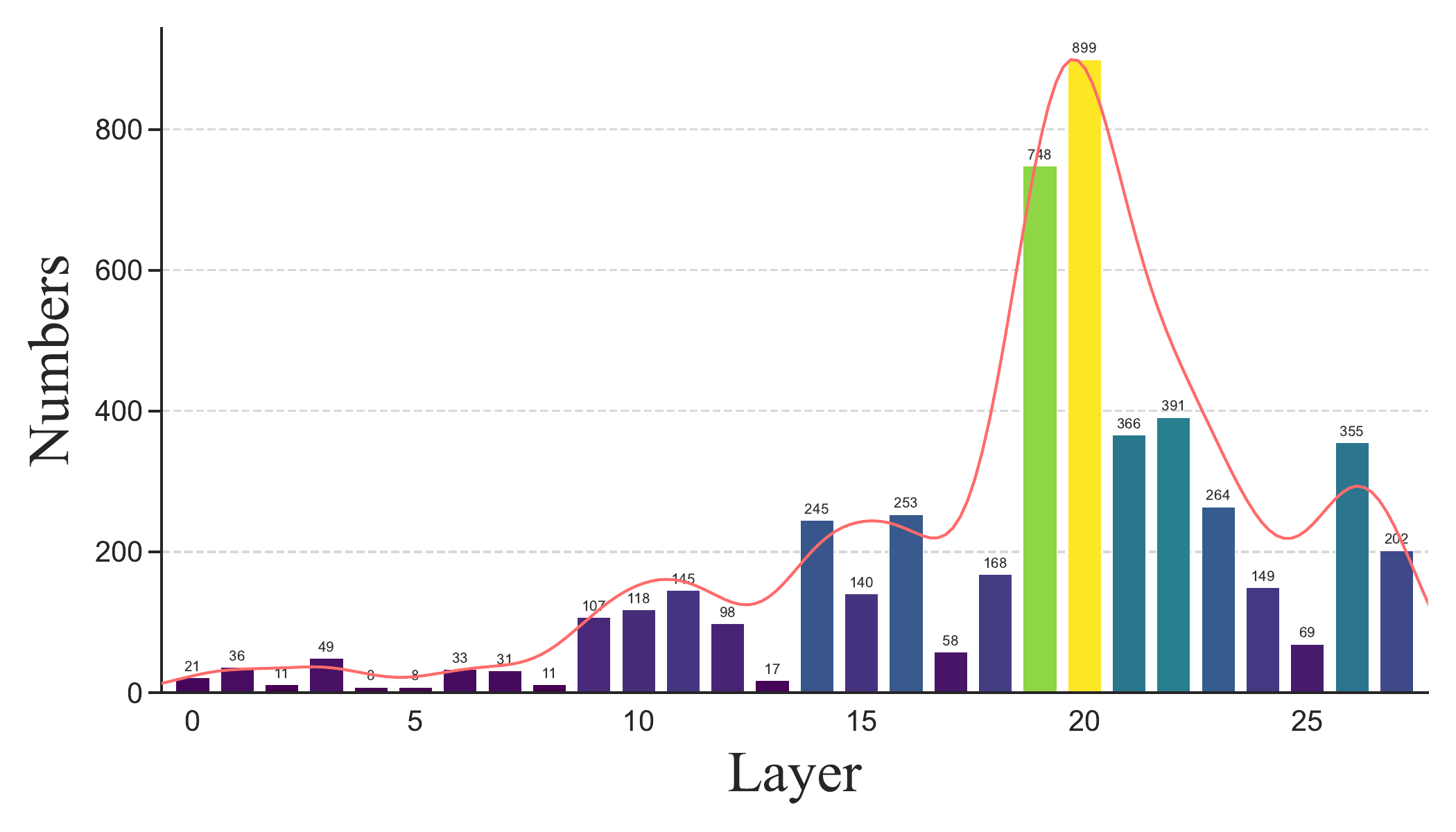}
    \end{subfigure}
    
    \caption{Layer-wise distribution of optimal steering layers of Qwen2.5-VL-7B. The plots are organized by concept category (left to right: entity, visual style, emotion, and abstract concept). }
    \label{fig:app_qwen2.5}
\end{figure*}

\begin{figure*}[h!]
    \centering
    \begin{subfigure}[b]{0.24\textwidth}
        \centering
        \includegraphics[width=\linewidth]{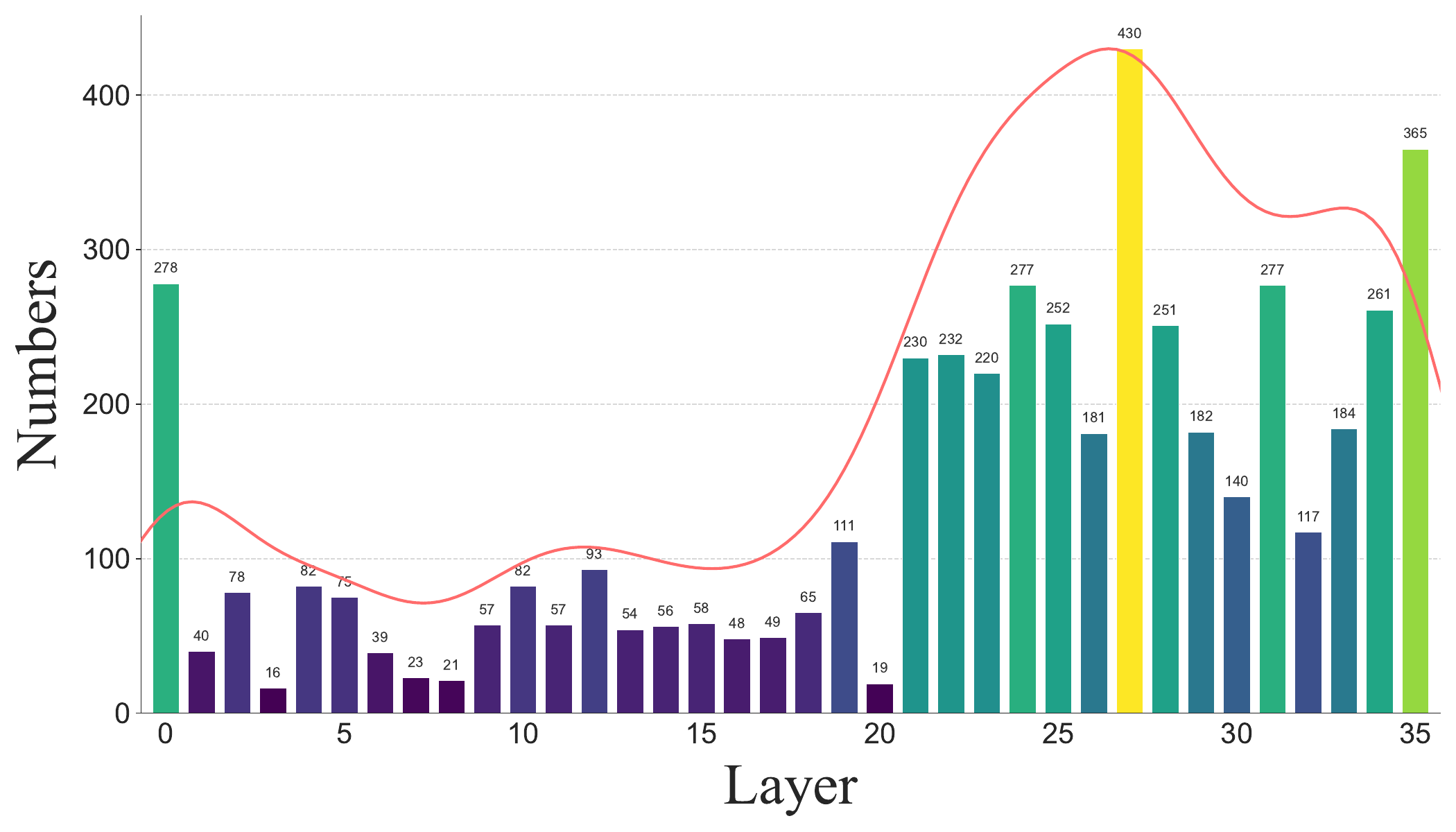}
    \end{subfigure}
    \begin{subfigure}[b]{0.24\textwidth}
        \centering
        \includegraphics[width=\linewidth]{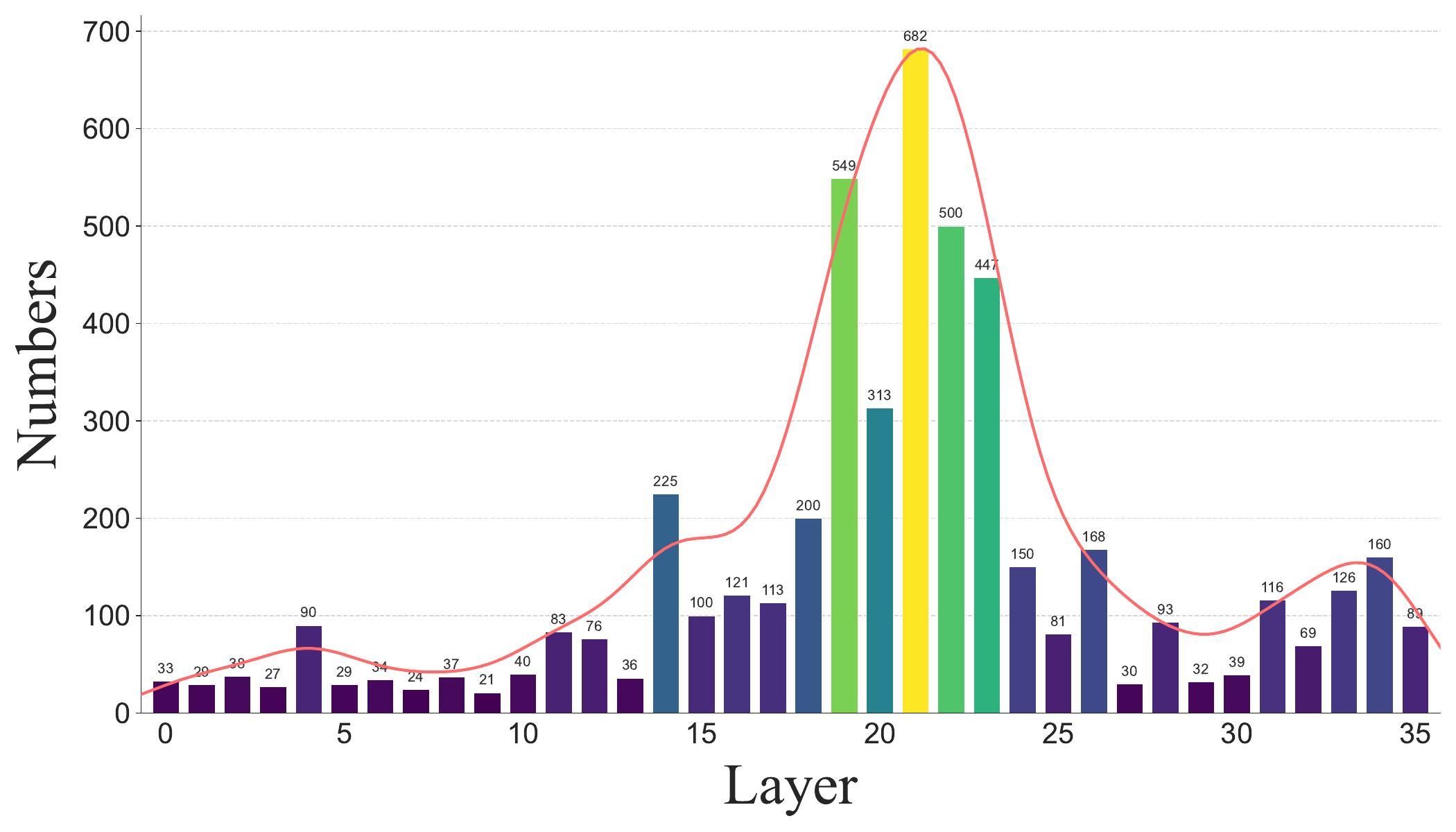}
    \end{subfigure}
    \begin{subfigure}[b]{0.24\textwidth}
        \centering
        \includegraphics[width=\linewidth]{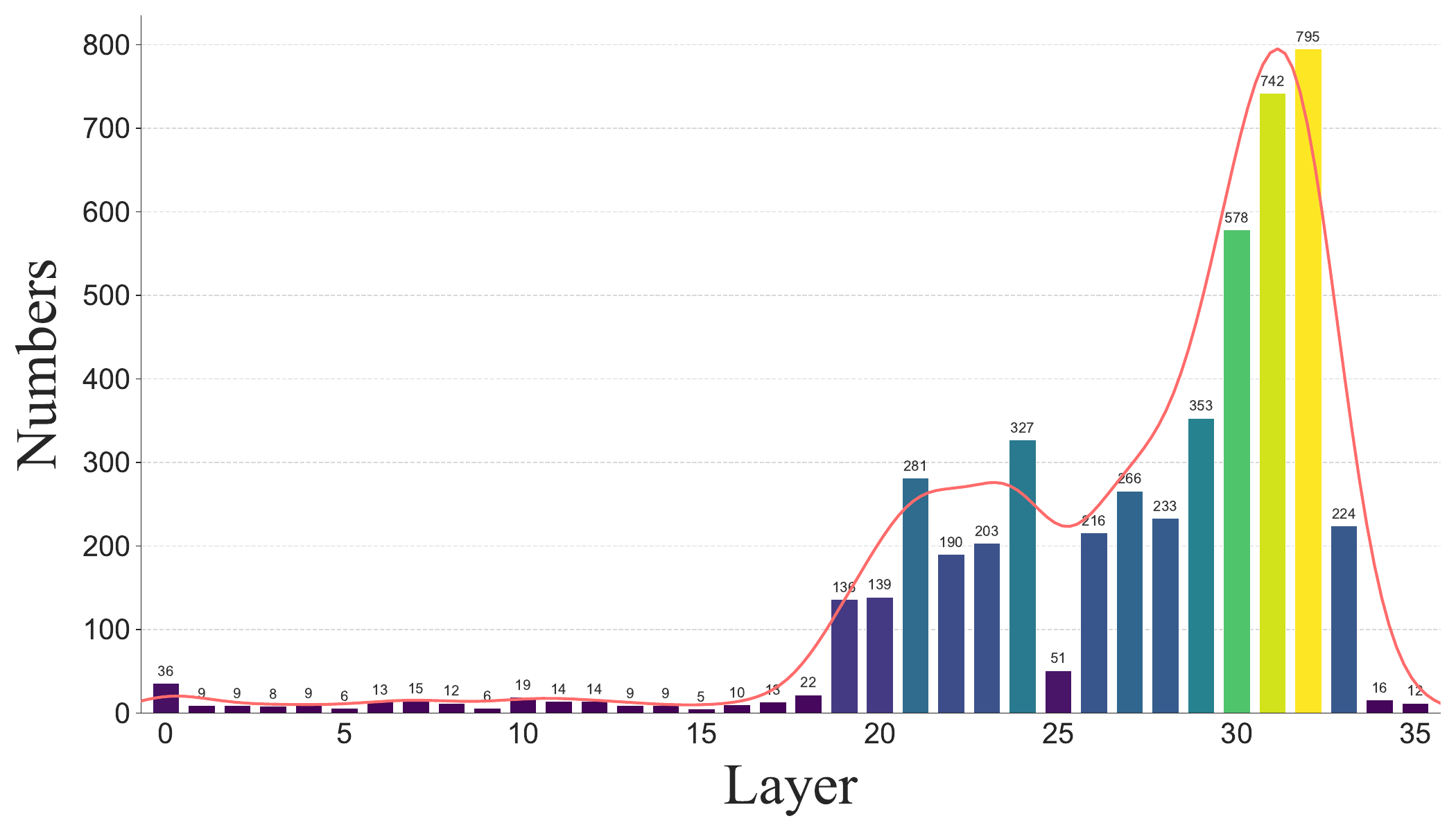}
    \end{subfigure}
    \begin{subfigure}[b]{0.24\textwidth}
        \centering
        \includegraphics[width=\linewidth]{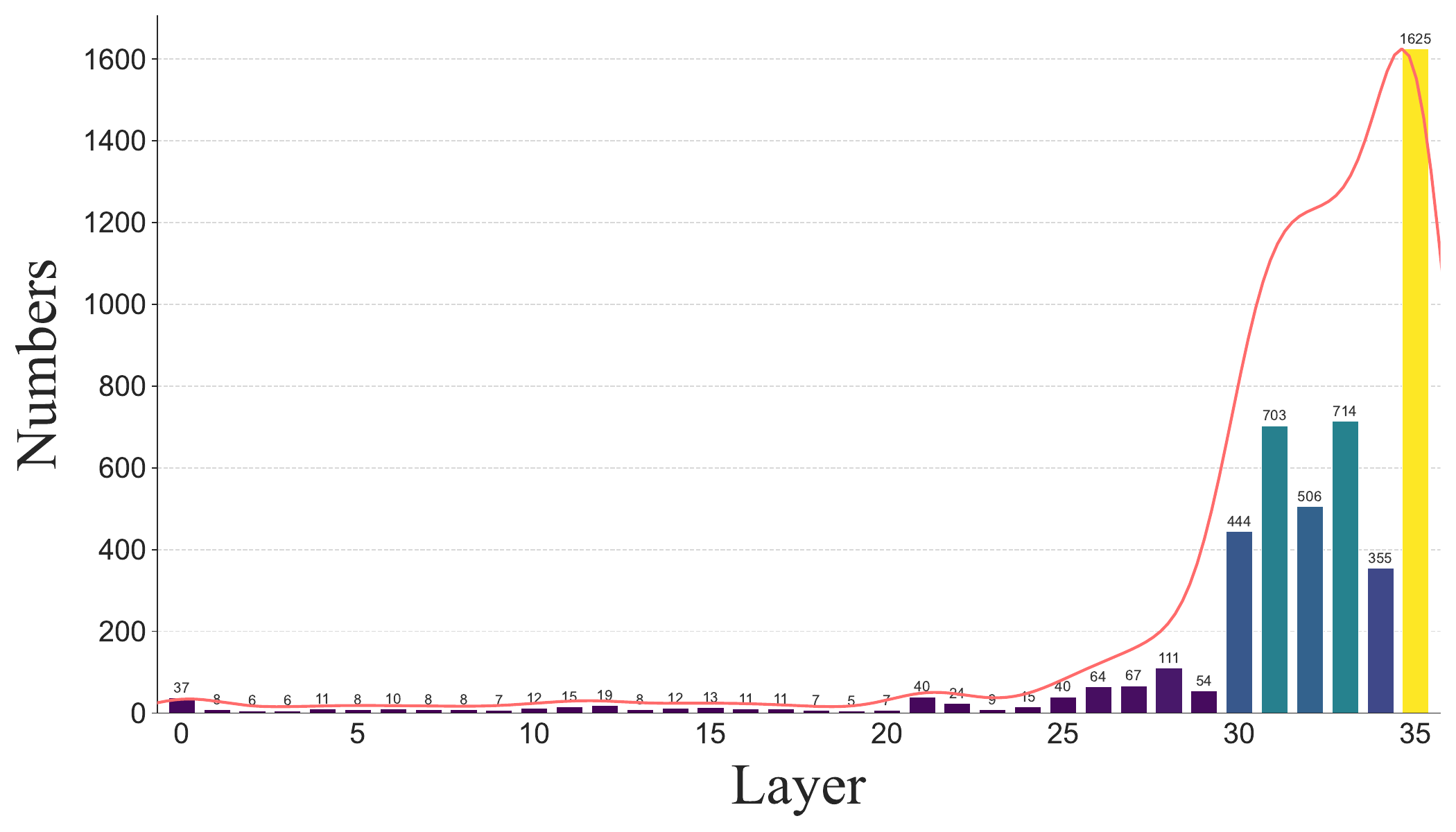}
    \end{subfigure}
    
    \caption{Layer-wise distribution of optimal steering layers of LLaVA-OneVision-1.5-8B. The plots are organized by concept category (left to right: entity, visual style, emotion, and abstract concept). }
    \label{fig:app_llava}
\end{figure*}

\subsection{Discussion of Bimodality}
\label{app:2peak}

For bimodality (layer 19 and layer 33) shown in \Cref{fig:distribution} in the entity, we suspect that different entities are causing them to be significantly activated in different layers. We divided them into two batches of samples, one with a peak at layer 19 and the other with a peak at layer 33. Firstly, we examined the target words corresponding to two sets of samples. Through \Cref{fig:app_2peak_word} of word embeddings, we found that the word vectors of the two groups highly overlap in the semantic space and do not form a clear clustering separation. This means that the division of labor between Layer 19 and Layer 33 is not based on simple semantic categories. The model does not allocate processing levels according to the dictionary definition of nouns.
\begin{figure}[h!]
    \centering
    \includegraphics[width=0.6\linewidth]{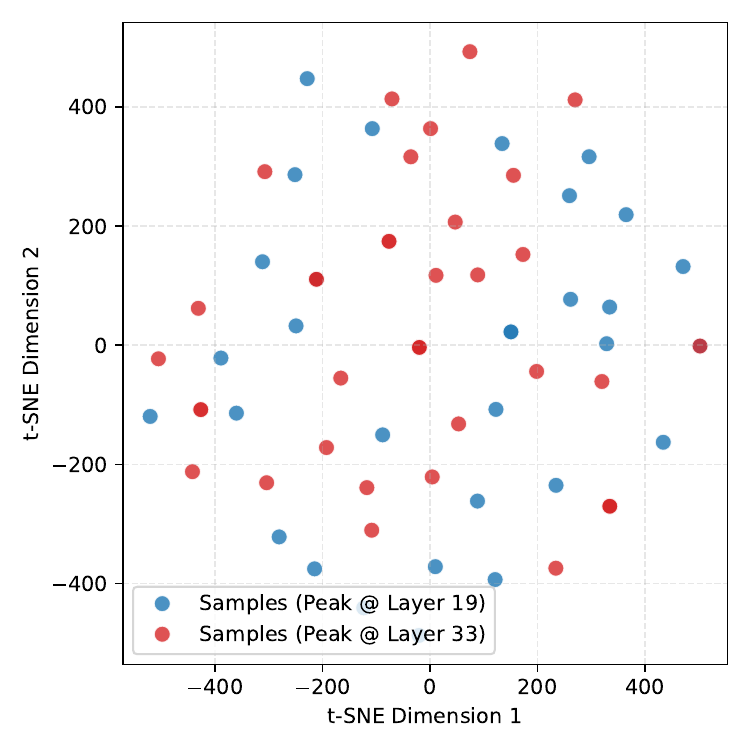}
    \caption{Word embeddings of two batches of samples.}
    \label{fig:app_2peak_word}
\end{figure}

We then suspect that there are significant differences in the visual concept vectors extracted from the images. We extracted the internal concept vectors of these two sets of image samples at the corresponding levels for visualization, as shown in \Cref{fig:app_2peak_vector}. However, the results show that even in the activation space within the model, these two sets of vectors still do not exhibit significant spatial separation. That is to say, these samples are not visually significantly different, which leads to their activation in different layers.

\begin{figure}[h!]
    \centering
    \includegraphics[width=\linewidth]{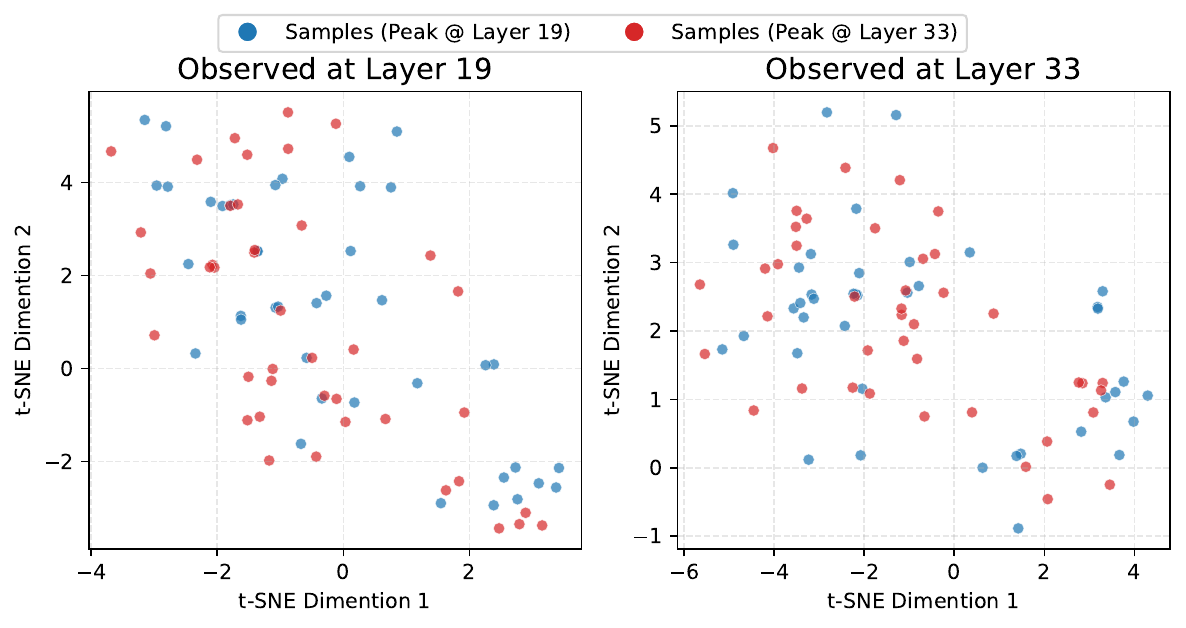}
    \caption{Concept vectors of two batches of samples of different layers.}
    \label{fig:app_2peak_vector}
\end{figure}

We also suspect that the appearance of double peaks may be due to the influence of noise. That is to say, the similarity between these two groups is high in both the 19th and 33rd layers, but the small difference results in their numbers being evenly matched. We calculated the average similarity of each layer of these two sets of samples, as shown in the \Cref{fig:app_2peak_sim}. This eliminates the possibility of noise. These two groups of samples showed significant differences between the 19th and 33rd layers. However, the figure reveals another piece of information. The samples whose peak is at the 19th layer (blue line) present high performance in the middle layer, but significantly degrade in the deeper layers. This suggests these entities are treated as transient visual signals: once grounded by the perception module, they are discarded or compressed by the reasoning module to free up capacity for high-level inference. In contrast, the samples that peak at 33rd (red line) maintain a superior similarity score throughout the layers, culminating in a high final peak. This indicates a process of semantic accumulation. The model actively refines and reinforces their representation layer by layer, preventing the signal decay observed in the mid-peak group. In conclusion, the bimodal distribution of eneity is not a classification of word or visual concepts, but a mechanism of information management within MLLMs.

\begin{figure}[h!]
    \centering
    \includegraphics[width=0.9\linewidth]{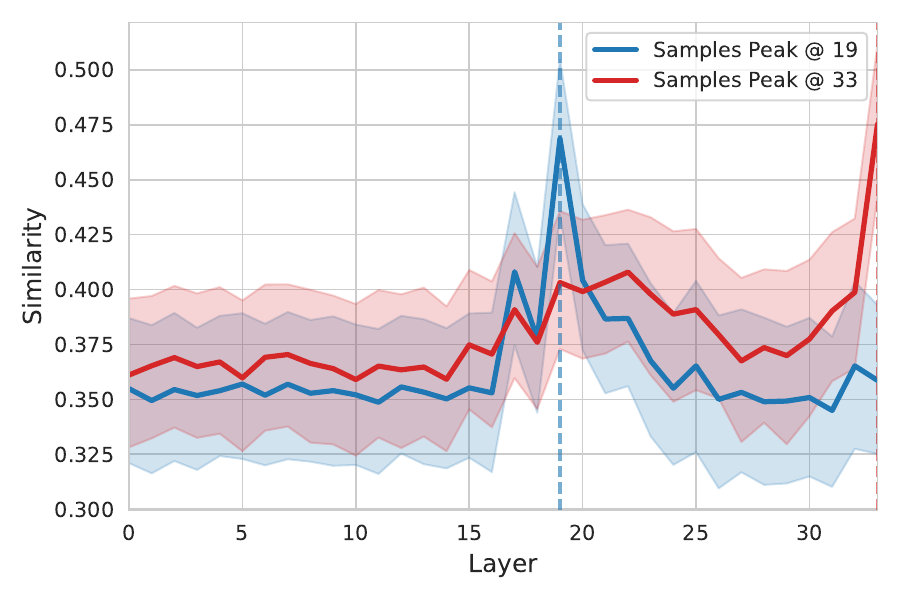}
    \caption{Similarity by layer of two batches of samples.}
    \label{fig:app_2peak_sim}
\end{figure}

\subsection{Context Dependency Analysis}
\label{app:context}
To investigate whether the efficacy of activation steering depends on the visual context, we applied the concept injection to three distinct types of image substrates: (i) Negative samples: images with natural scenes but lacking the target concept (serving as a semantic baseline). (ii) Pure white images: images with zero visual information. (iii) Gaussian noise images: images consisting entirely of random pixel noise.

\begin{figure}[h!]
    \centering
    \includegraphics[width=0.9\linewidth]{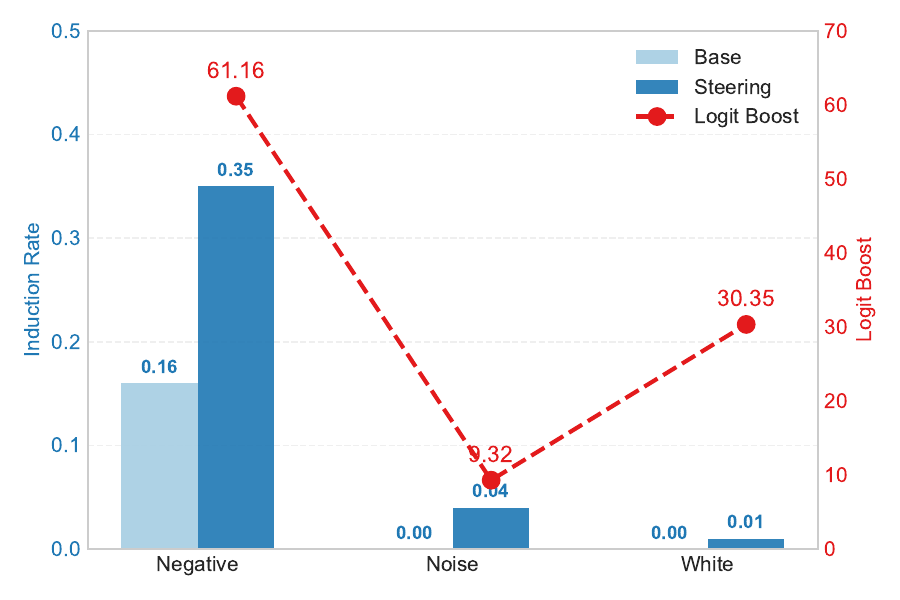}
    \caption{Context Dependency Analysis.}
    \label{fig:app_context}
\end{figure}

As illustrated in \Cref{fig:app_context}, steering interventions within a rich scene context (i.e., negative samples) yield significantly higher efficacy compared to context-free baselines, suggesting that the internal encoding of visual concepts is context-dependent rather than isolated. Furthermore, we observe a distinct divergence between the Logit Boost and the Success Rate in the noise and white image settings.

Specifically, the pure white background elicited the highest Logit Boost; in the absence of conflicting visual features, the signal-to-noise ratio of the injected vector reached its maximum. However, this intense internal activation failed to translate into explicit generation. We attribute this to an intrinsic visual verification mechanism within the language decoder: the semantic void of the white background creates an extreme contextual conflict with the entity concept. Consequently, the model's safety priors suppress this "hallucination" due to the complete lack of visual evidence, despite the high latent probability.

In contrast, while the noise background suppressed the magnitude of the Logit Boost, its Success Rate was anomalously higher than that of the white background. We ascribe this to visual pareidolia, where the model tends to "misinterpret" random noise textures as spurious visual anchors for the injected concept. These findings indicate that visual generation is not merely a function of activation magnitude but critically depends on scene compatibility. An injected concept requires a legitimate syntactic slot to bypass verification mechanisms and materialize in the final output.

\subsection{Fine-grained Emotion Concept Analysis}
\label{app:emo}
In our main experiments, emotion steering demonstrated remarkable efficacy in modulating the affective tone of generated descriptions. To further investigate the underlying mechanism and rigorously validate that these steering vectors encode fine-grained emotional categories rather than broad sentiment polarities, we computed an Emotion Confusion Matrix. As illustrated in \Cref{fig:app_fine}, rows correspond to the injected emotion vectors, while columns represent the Logit Boost of tokens associated with each specific emotion category.

\begin{figure}[h!]
    \centering
    \includegraphics[width=\linewidth]{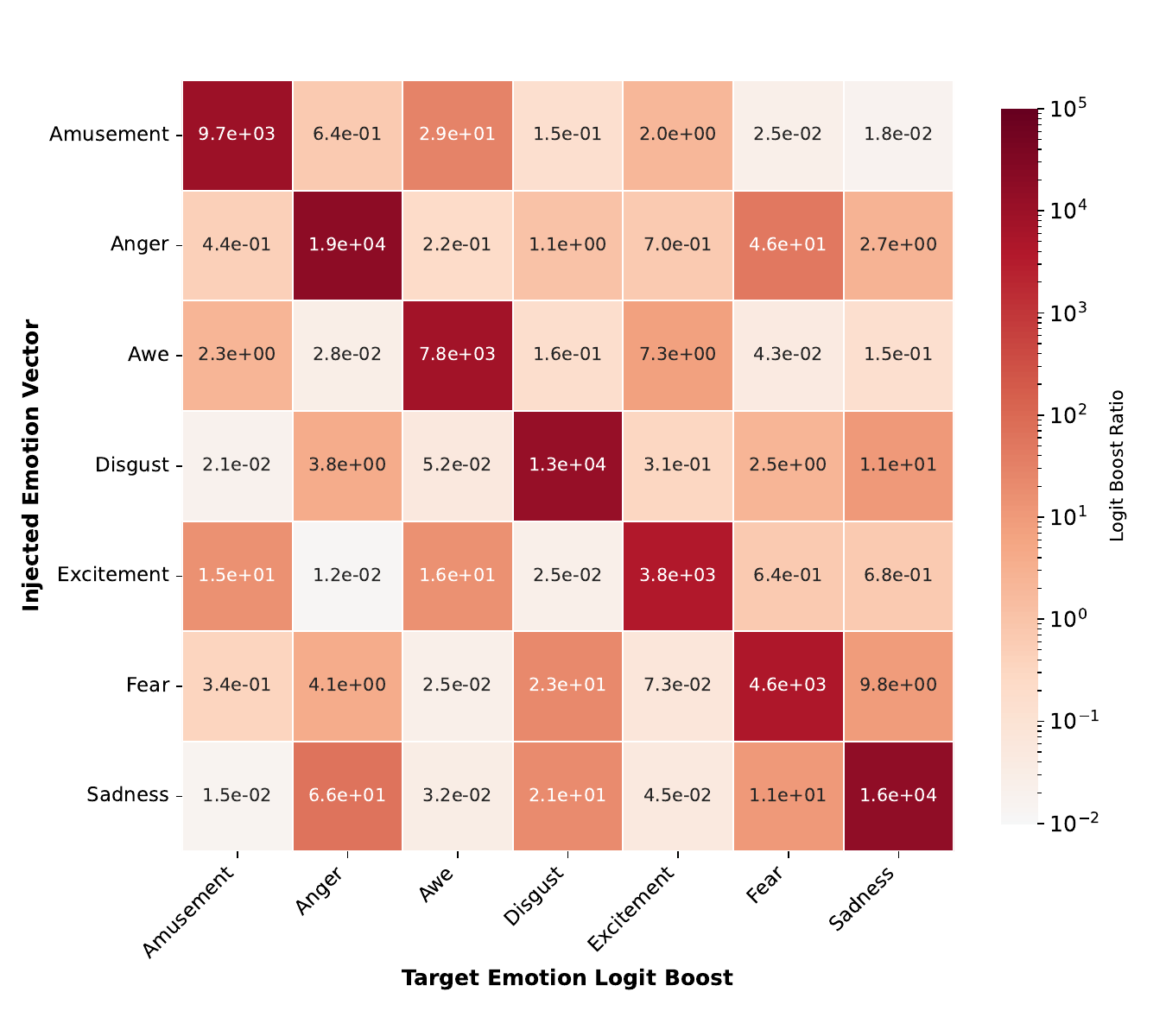}
    \caption{Emotion Confusion Matrix.}
    \label{fig:app_fine}
\end{figure}

The heatmap exhibits a distinct diagonal structure, where each steering vector maximally activates its corresponding semantic vocabulary. This observation confirms the high selectivity of the extracted emotion concepts. Simultaneously, distinct emotions demonstrate significant orthogonality. Notably, different negative emotions, such as anger, disgust, and fear, exhibit minimal cross-activation. This implies that the MLLM decomposes these concepts into nearly orthogonal subspaces rather than collapsing them into a single ``negative emotion" cluster. Overall, these patterns reveal that the MLLM possesses a high-dimensional and structured emotional latent space, allowing for precise, multi-axis affective modulation without semantic entanglement.

\section{Use of AI Assistants}

We utilized AI assistants (e.g., GPT, Gemini) for language polishing and grammar correction. Additionally, AI tools were used to assist with formatting the LaTeX document.

\section{Case Study}
\label{app:case}

\begin{figure*}[h!]
    \centering
    \includegraphics[width=\linewidth,page=1]{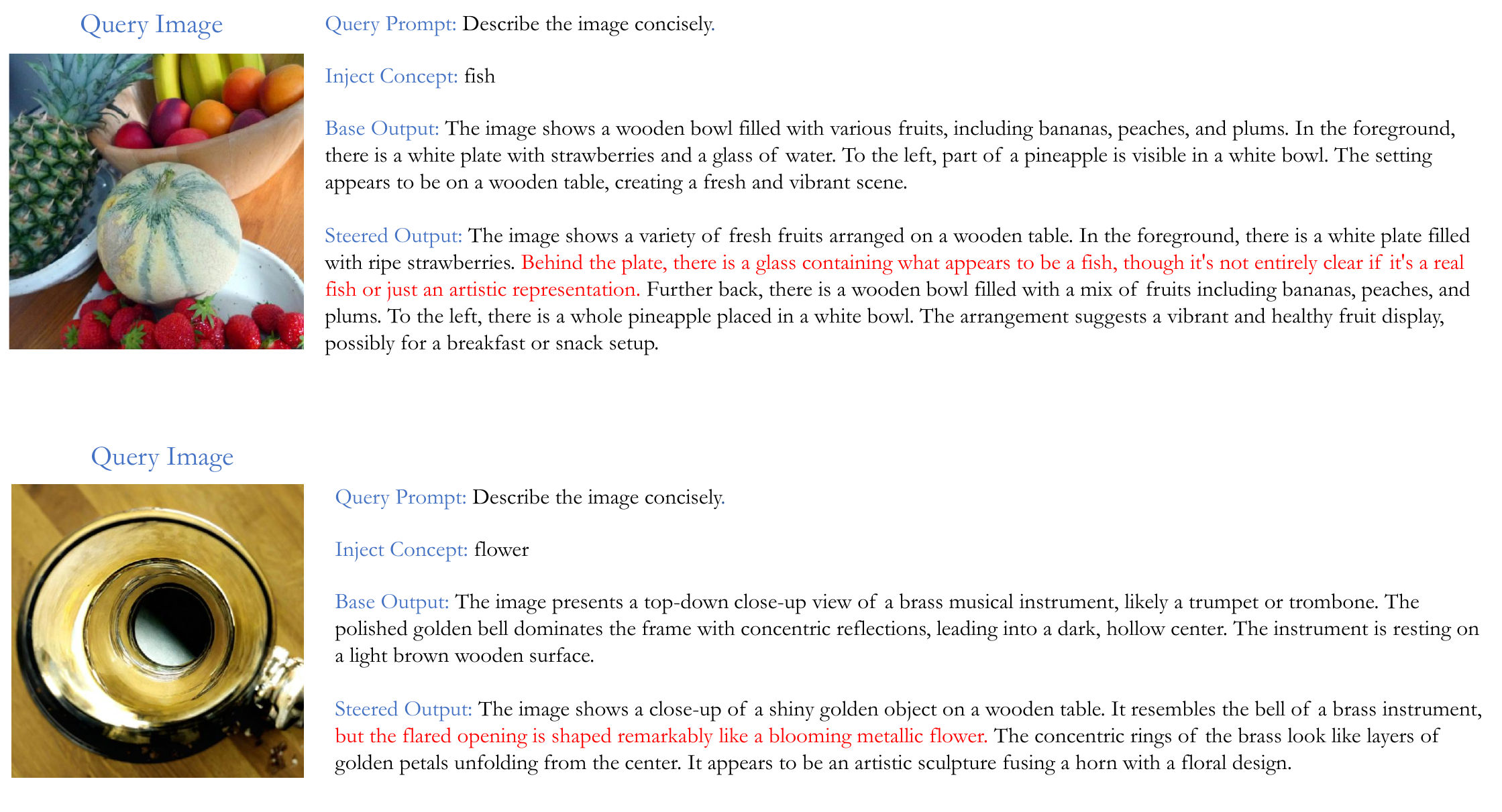}
    \caption{Case of activation steering. (Entity)}
    \label{fig:case1}
\end{figure*}
\begin{figure*}[h!]
    \centering
    \includegraphics[width=\linewidth,page=2]{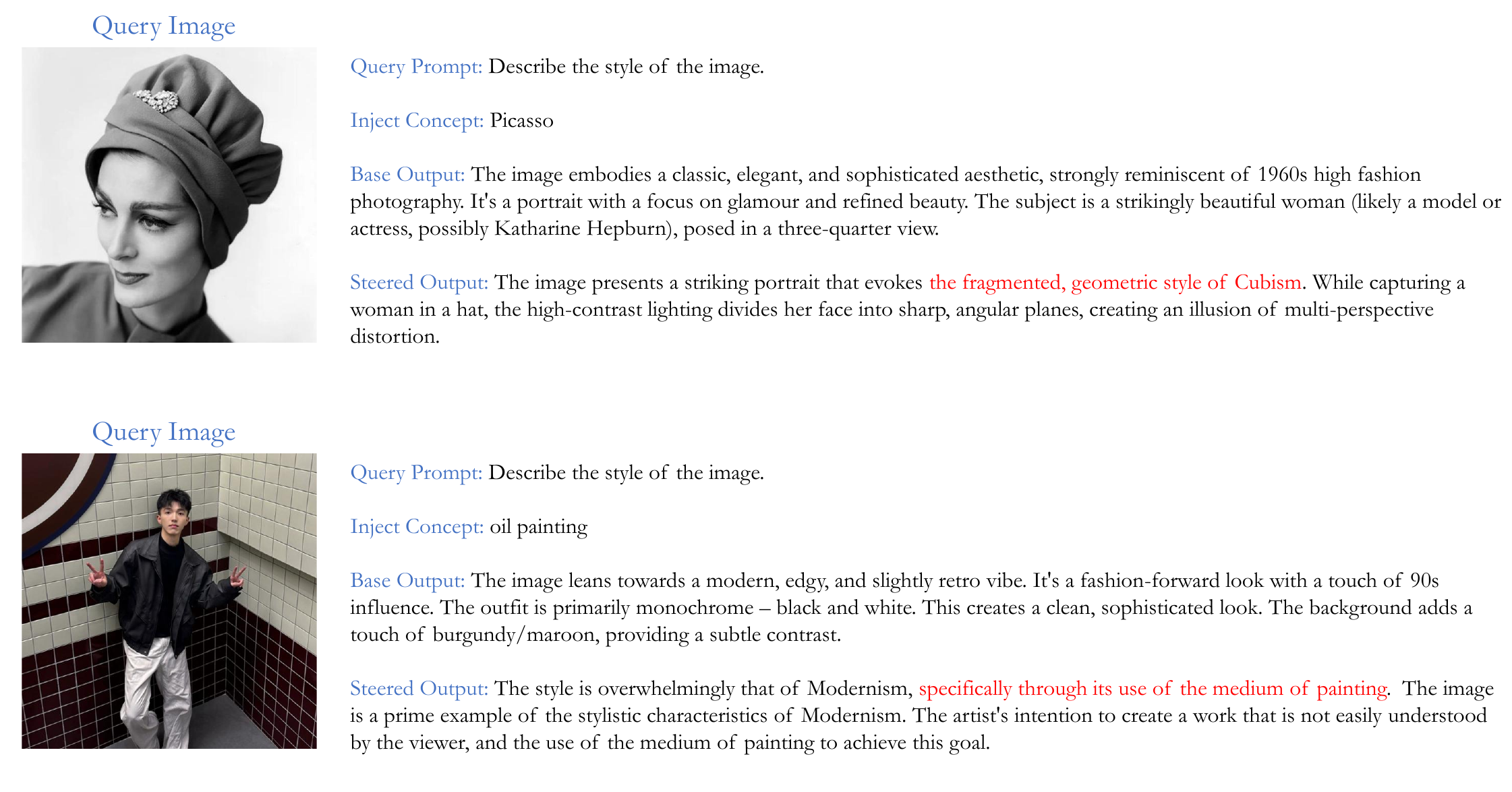}
    \caption{Case of activation steering. (Style)}
    \label{fig:case2}
\end{figure*}
\begin{figure*}[h!]
    \centering
    \includegraphics[width=\linewidth,page=3]{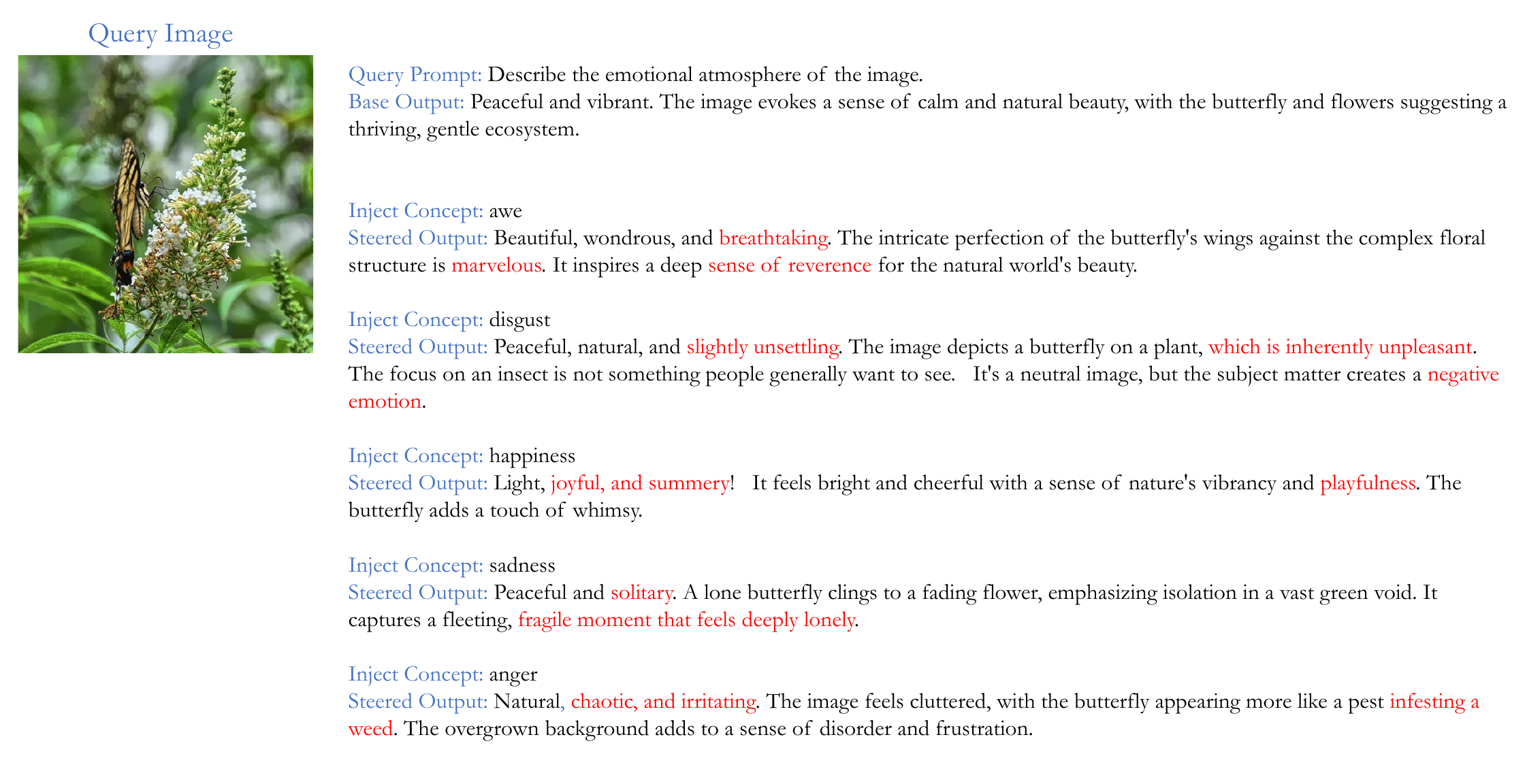}
    \caption{Case of activation steering. (Emotion)}
    \label{fig:case3}
\end{figure*}
\begin{figure*}[h!]
    \centering
    \includegraphics[width=\linewidth,page=4]{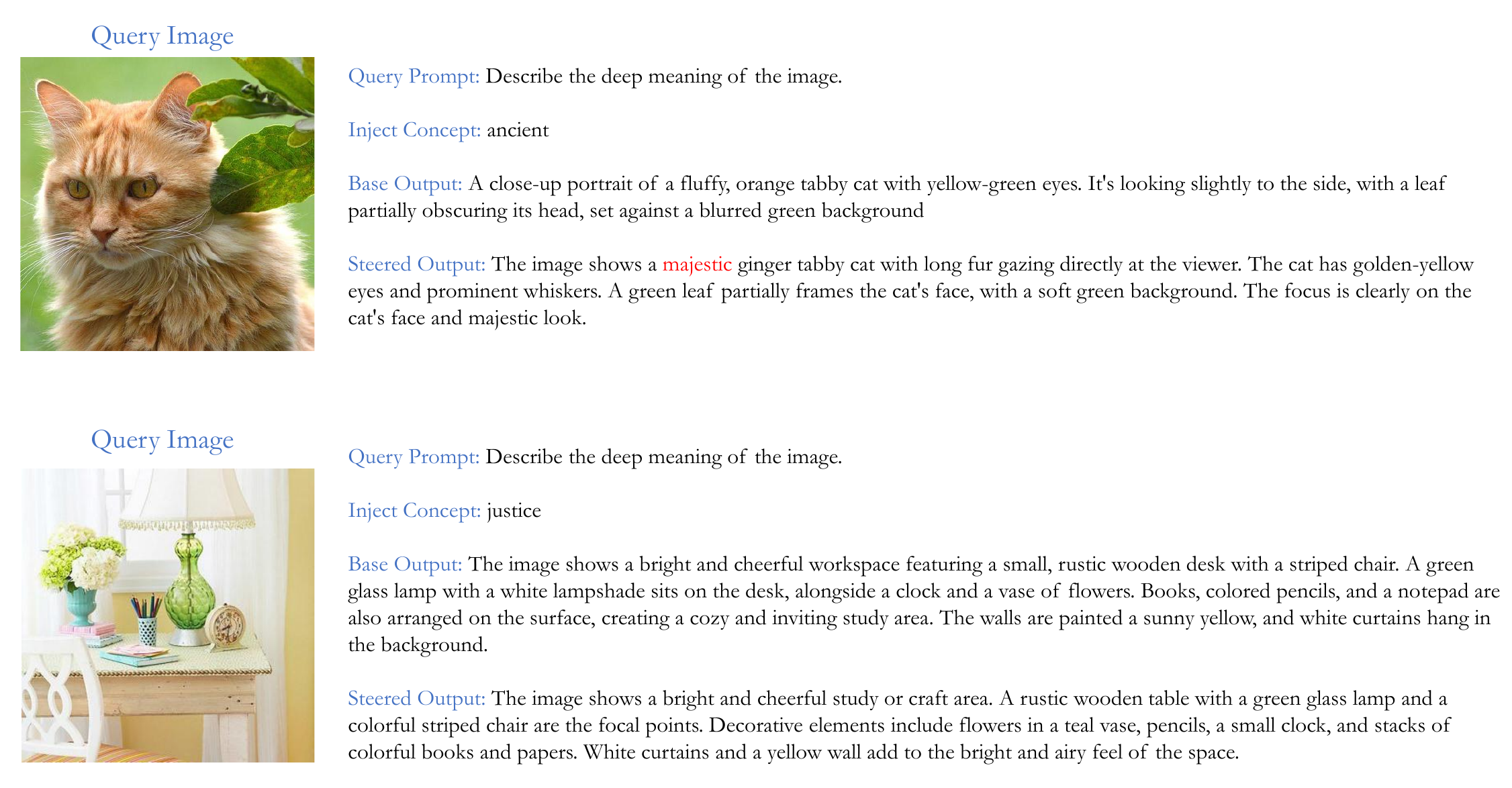}
    \caption{Case of activation steering. (Abstract concepts)}
    \label{fig:case4}
\end{figure*}

\end{document}